\documentclass[lettersize,journal]{IEEEtai}
\usepackage{amsmath,amsfonts}
\usepackage{algorithmic}
\usepackage{algorithm}
\usepackage{array}
\usepackage[caption=false,font=normalsize,labelfont=sf,textfont=sf]{subfig}
\usepackage{textcomp}
\usepackage{stfloats}
\usepackage{url}
\usepackage{verbatim}
\usepackage{graphicx}
\usepackage{cite}
\usepackage{bm}
\usepackage{makecell}
\usepackage{amsfonts}  
\usepackage{algorithm}
\usepackage{algorithmic}
\usepackage{siunitx}   
\usepackage{upgreek}   
\usepackage{xcolor}    
\usepackage{booktabs}       
\usepackage{longtable}
\usepackage{threeparttable}
\usepackage{multirow}
\usepackage{multicol}
\usepackage{times}
\usepackage{hyperref}
\let\labelindent\relax
\usepackage{enumitem}

\usepackage{amsthm}

\theoremstyle{plain}
\newtheorem{theorem}{Theorem}
\newtheorem{proposition}[theorem]{Proposition}
\newtheorem{lemma}[theorem]{Lemma}

\theoremstyle{definition}
\newtheorem{definition}[theorem]{Definition}

\theoremstyle{remark}

\def\x{\boldsymbol{x}}

\begin{document}
	
\title{\huge\bf On Expressivity of Height in Neural Networks}

\author{Feng-Lei Fan$^{1}$, \textit{Member, IEEE}, Ze-Yu Li$^{1}$, Huan Xiong$^{2}$,
Tieyong Zeng$^{1*}$ 
\thanks{*This author is corresponding author. All authors contribute equally. The order of author names is alphabetical.}
\thanks{$^{1}$Feng-Lei Fan (hitfanfenglei@gmail.com), Ze-Yu Li, and Tieyong Zeng are with Department of Mathematics, The Chinese University of Hong Kong, Hong Kong }
\thanks{$^{2}$Huan Xiong is with the Institute of Advanced Mathematics, Harbin Institute of Technology, Harbin, Heilongjiang Province, China} 
}

\markboth{Journal of \LaTeX\ Class Files,~Vol.~14, No.~8, August~2021}%
{XXX \MakeLowercase{\textit{et al.}}: On Expressivity of Height in Neural Networks}

	
\maketitle
	
\begin{abstract}
In this work, beyond width and depth, we augment a neural network with a new dimension called height by intra-linking neurons in the same layer to create an intra-layer hierarchy, which gives rise to the notion of height. We call a neural network characterized by width, depth, and height a 3D network. To put a 3D network in perspective, we theoretically and empirically investigate the expressivity of height. We show via bound estimation and explicit construction that given the same number of neurons and parameters, a 3D ReLU network of width $W$, depth $K$, and height $H$ has greater expressive power than a 2D network of width $H\times W$ and depth $K$, \textit{i.e.}, $\mathcal{O}((2^H-1)W)^K)$ vs $\mathcal{O}((HW)^K)$, in terms of generating more pieces in a piecewise linear function. Next, through approximation rate analysis, we show that by introducing intra-layer links into networks, a ReLU network of width $\mathcal{O}(W)$ and depth $\mathcal{O}(K)$ can approximate polynomials in $[0,1]^d$ with error $\mathcal{O}\left(2^{-2WK}\right)$, which improves $\mathcal{O}\left(W^{-K}\right)$ and $\mathcal{O}\left(2^{-K}\right)$ for fixed width networks. Lastly, numerical experiments on 5 synthetic datasets, 15 tabular datasets, and 3 image benchmarks verify that 3D networks can deliver competitive regression and classification performance.
\end{abstract}

\begin{IEEEImpStatement}
Due to the constraint of hardware, sustaining the law of "the deeper the better" is increasingly difficult. We propose a new view of network design by introducing the so-called height into a network's topology. Height is based on
intra-layer links, which were already seen in temporal models and other miscellaneous places. However, this work generalizes intra-layer links from temporal models to generic foundational models and interprets intra-layer links as a new dimension of height, which provides a unique and systematic framework for designing a network. Both theoretical and empirical experiments show that the height network is expressive, and can deliver the superior performance in machine learning tasks.
\end{IEEEImpStatement}
	
	
\section{INTRODUCTION}

\IEEEPARstart{I}n the last eighties, the popular network architectures are shallow due to the constraint of computing resources. Over the past decade, deep networks such as ResNet and Transformer have made lots of successes in many important fields \cite{lecun2015deep}, giving the notion that “the deeper the better". We call shallow and deep networks 1D and 2D networks, respectively. The development of deep/large networks has greatly increased the cost of training and deployment. The GPU cluster that costs millions of dollars is needed to train a large model. Thus, the deeper/larger models are more and more dominated by large institutes, which distort the ecosystem of AI. Hardware-wise, the most advanced GPUs will soon be made by the 1nm fabrication, meeting its physical ceiling. Therefore, the trend of “the deeper the better" is hard to be sustained in the foreseeable future. What may future architectures look like?

In this paper, beyond width and depth, we propose a novel network by augmenting one more dimension, referred to as height. We call networks characterized by width ($W$), depth ($K$), and height ($H$) 3D networks. Specifically, height is introduced by stacking additional neurons in a layer of standard 2D networks and intra-linking new neurons with the original ones, as Figure \ref{Figure_intra_linked} shows. We call links connecting new neurons and the original within a layer intra-layer links. Intra-layer links realize interconnections within a layer such that new neurons cannot be straightened out, meaning that a hierarchy is constructed within a layer. Such a hierarchy naturally generates a notion of height. In other words, inserting intra-layer links in a layer can realize the 2D-3D transformation of a network. 

\begin{figure}[h]
\vspace{-0.1cm}
\center{\includegraphics[width=\linewidth] {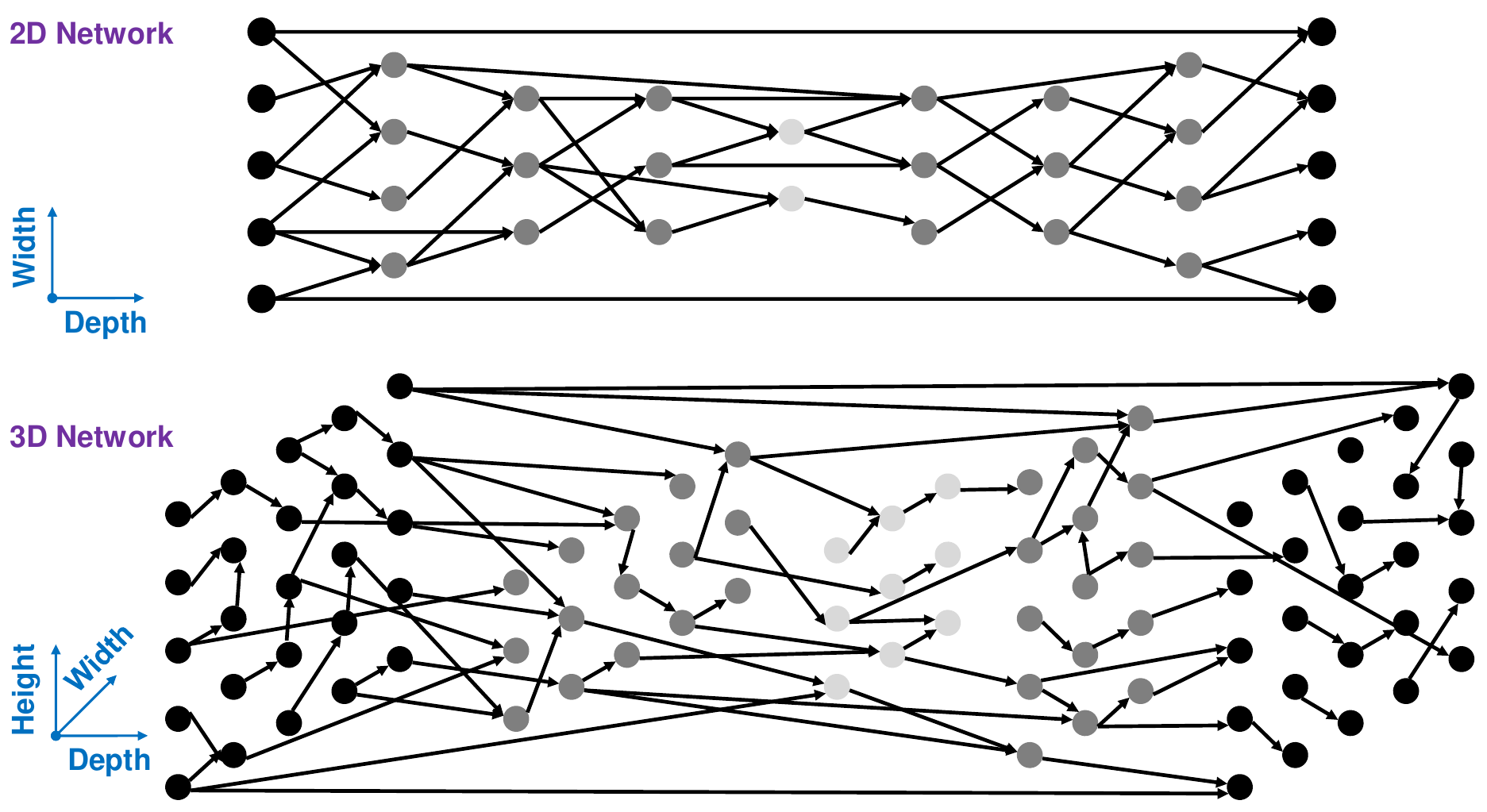}}
\caption{(a) A 2D network characterized by width and depth. (b) A 3D network characterized by width, depth, and height.}
\vspace{-0.3cm}
\label{Figure_intra_linked}
\end{figure}

Temporal models such as Recurrent Neural Networks (RNNs) \cite{jordan1997serial}, Long Short-Term Memory (LSTM) \cite{hochreiter1997long}, and Gated Recurrent Units (GRUs) \cite{cho2020learning} also utilize the intra-layer links to connect hidden states at different timestamps, which is a natural design for sequential data. In contrast, we assert that intra-layer connections can be leveraged for a broader range of data types and serve as a critical component of foundational models. Also, concurrent with this work, \cite{zhang2022theoretically} shows that using intra-layer links can improve the generalization of spiking networks. \cite{sadat2023connected} confirms that using intra-layer links can result in rapid convergence. Unlike these studies, our approach strategically elevates intra-layer links from a mere art to a distinctive and systematic perspective on dimensional augmentation for foundational network design. This perspective injects new momentum into the utilization of intra-layer connections, promoting innovative developments in network architecture.

To the best of our knowledge, height was few studied in neural networks, except for \cite{zhang2022neural} which used a trainable compositional network as the activation of a layer and took the number of compositions as the height of the overall network. We argue that such a definition of height is based on an abstract view, while height is more naturally introduced by linking neurons in a layer.  
Theoretically, a 3D network can be expanded into a much deeper 2D network. However, the number of parameters will increase many times for computing because augmenting the dimension of height provides a special and efficient parameter-sharing scheme. For example, features of neurons in a layer not only serve neurons in succeeding layers but also are used in neurons in the same layer that intra-link them. It will be tedious to represent neurons that serve multiple places.

We show the benefits of height from the perspective of approximation theory by addressing the key question: why is a 3D network more expressive than a 2D network given the same number of neurons and parameters in a layer? We focus on networks using the rectified linear units (ReLU). Our theoretical results are twofold: 1) Through bound analysis and explicit construction, we prove that given the same number of neurons, a 3D network can produce many more pieces than a standard 2D network, and the gain is at most exponential. This means that the augmented dimension of height can boost the expressive ability of a network without increasing the number of parameters in a network. 2) We further compare the universal approximation error of 2D and 3D networks. Assume that in a 2D network, $W$ is the width of a network, and $K$ is the depth of a network, while in the corresponding 3D network (all neurons in a layer are intra-linked), the height equals $W$, the width is $1$, and the depth is $K$. By introducing intra-layer links into networks, a 3D ReLU network can approximate polynomials in $[0,1]^d$ with the error of $\mathcal{O}\left(2^{-2WK}\right)$, which is a non-trivial extension of the result $\mathcal{O}\left(W^{-K}\right)$ in \cite{shen2019deep} and $\mathcal{O}\left(2^{-K}\right)$ in \cite{yarotsky2017error} for fixed width networks. This also means that without increasing the number of parameters, 2D-3D transformation via intra-layer links can bring great improvement in approximation power. Furthermore, encouraged by the theoretical analysis, we empirically confirm the power of 3D networks with systematic experiments on 5 synthetic datasets, 15 tabular datasets, and 3 image benchmarks. To summarize, our contributions are threefold:

\begin{table}[!t]
\centering\footnotesize
\caption{Comparison between the approximation ability of 2D and 3D networks. In a 2D network, $W$ is the width of a network, and $K$ is the depth of a network. In a 3D network (all neurons in a layer are intra-linked), the height equals $W$, the width is $1$, and the depth is $K$.}
\label{Table:summary_of-separation}
\renewcommand{\arraystretch}{1.2}
\setlength{\tabcolsep}{5pt}
\begin{tabular}{c|c|c}
\hline
 & 2D & 3D (With Intra-layer Links) \\
 \hline
\#Pieces & $\mathcal{O}\left((W+1)^K\right)$ & $\mathcal{O}\left((2^W-1)W+1)^K\right)$ \\
\hline
Error Rate & $\mathcal{O}(W^{-K})$, $\mathcal{O}(2^{-K})$  & $\mathcal{O}(2^{-2WK})$ \\
\hline
\end{tabular}
\vspace{-0.4cm}
\end{table}

\begin{enumerate}
    \item We propose a 3D neural network by augmenting the dimension of height using intra-layer links, which is a valuable addition to the neural network family.
    \item We theoretically prove that without increasing the number of parameters, using intra-layer links can greatly boost the representation power of a network in terms of the number of pieces and approximation rate.
    \item We use systematic experiments to demonstrate the empirical advantages of 3D networks.
\end{enumerate}

\section{RELATED WORK}

\textbf{The power of width}. In the early days, it was proved that given an arbitrary number of neurons (width), a one-hidden-layer network could approximate any continuous function \cite{hornik1989multilayer}. Recently, the power of width has been revisited. \cite{fan2020quasi} showed that the width and depth of a ReLU network can be converted to each other. \cite{lu2017expressive, park2020minimum} demonstrated that the minimum width of a network to learn a function from $\mathbb{R}^m$ to $\mathbb{R}^n$ is $m+n$. Levine et al. \cite{levine2020limits} showed that widening is necessary when deepening a transformer; otherwise, the performance of the model cannot fulfill the expectation. In brief, the width and depth should be balanced according to these studies. 
 
\textbf{The power of depth}. Recently, exploring the approximation ability of a deep network has been increasingly popular. The approximation rates of deep networks for different classes of functions were investigated such as 1-Lipschitz continuous functions \cite{shen2020deep, yarotsky2018optimal}, smooth functions \cite{lu2021deep, yarotsky2020phase}, and band-limited functions \cite{montanelli2021deep}. Also, it was found that designing novel activation functions to replace ReLU can achieve the superior approximation rate such as Floor-ReLU networks \cite{shen2021deep}, Elementary Unit Activation Function \cite{zhang2022deep}, Floor-Exponential-Step function \cite{shen2021neural}, and $\{\sin, \arcsin\}$\cite{yarotsky2021elementary}. Furthermore, it was proved that introducing multiplicative operations can greatly enlarge the hypothesis space \cite{jayakumar2019multiplicative, fan2020universal} and boost the expressivity of a network \cite{fan2023expressivity}.

Due to the widespread applications of deep networks in many important fields \cite{lecun2015deep}, mathematically understanding the power of deep networks has been a central problem in deep learning theory \cite{poggio2020theoretical}. The key issue is figuring out how expressive a deep network is or how increasing depth promotes the expressivity of a neural network better than increasing width. In this regard, there have been a plethora of studies on the expressivity of deep networks \cite{safran2019depth,vardi2020neural,guhring2020expressivity,vardi2021size,safran2022optimization, venturi2022depth, vardi2022width}. 

A popular idea in depth theory is the complexity characterization that introduces appropriate complexity measures for functions represented by neural networks \cite{pascanu2013number, montufar2014number, telgarsky2015representation, montufar2017notes,serra2018bounding,hu2018nearly,xiong2020number,bianchini2014complexity, raghu2017expressive,sanford2022expressivity,joshi2023expressive}, and then reports that increasing depth can greatly boost such a complexity measure. The complexity analysis is to characterize the complexity of the function represented by a neural network, thereby demonstrating that increasing depth can greatly maximize such a complexity measure. Currently, one of the most popular complexity measures is the number of linear regions because it conforms to the functional structure of the widely-used ReLU networks. For example, 
\cite{pascanu2013number, montufar2014number,  montufar2017notes,serra2018bounding, hu2018nearly, hanin2019deep} estimated the bound of the number of linear regions generated by a fully-connected ReLU network by applying Zaslavsky’s Theorem \cite{zaslavsky1997facing}.
\cite{xiong2020number} offered the first upper and lower bounds of the number of linear regions for convolutional networks. Other complexity measures include classification capabilities \cite{malach2019deeper}, Betti numbers \cite{bianchini2014complexity}, trajectory lengths \cite{raghu2017expressive},  global curvature \cite{poole2016exponential}, and topological entropy \cite{bu2020depth}.
Please note that using complexity measures to justify the power of depth demands a tight bound estimation. Otherwise, it is insufficient to say that shallow networks cannot be as powerful as deep networks, since deep networks cannot reach the upper bound.


The construction analysis is to find an explicit family of functions that are hard to approximate by a shallow network, but can be efficiently approximated by a deep network. \cite{eldan2016power} built a special radial function that is expressible by a 3-layer
neural network with a polynomial number of neurons, but a 2-layer network can do the same level approximation only with an exponential number of neurons. Later, \cite{safran2017depth} extended this result to a ball function, which is a more natural separation result. \cite{venturi2021depth} generalized the construction of this type to a non-radial function.
\cite{telgarsky2015representation, telgarsky2016benefits} used an $\mathcal{O}(k^2)$-layer network to construct a sawtooth function. Given that such a function has an exponential number of pieces, it cannot be expressed by an $\mathcal{O}(k)$-layer network, unless the width is $\mathcal{O}(\exp(k))$. \cite{arora2016understanding} estimated the maximal number of pieces a network can produce, and established the size-piece relation to advance the depth separation results from ($k^2$, $k$) to ($k$, $k'$), where $k'<k$. \cite{daniely2017depth} proved that poly-size depth neural networks with (exponentially) bounded weights cannot approximate $f: \mathbb{S}^{d-1}\times \mathbb{S}^{d-1} \to \mathbb{R} $ which has the form $f(\x,\x')=g(\langle \x,\x'\rangle)$ whenever $g$ cannot be expressed by a low-degree polynomial. 
Other smart constructions include polynomials \cite{rolnick2017power}, functions of a compositional structure \cite{poggio2017and}, Gaussian mixture models \cite{jalali2019efficient}, and among others. In a broad sense, we summarize the elements of establishing a depth separation theorem as the following: i) there exists a function representable by a deep network; ii) such a function cannot be represented by a shallow network whose width is lower than a large threshold. Recently, \cite{malach2019deeper} explored the relationship between the expressive properties of a deep network and the trainability using gradient descent-based methods.
\cite{vardi2022width} proved that there are no functions that can be expressed by wide and shallow neural networks but cannot be approximated by a narrow but deep network.

We do not contend that depth is no longer important when a network has height. Instead, our work still can be regarded as the modification to the existing network. We assume that when a wide and deep network is given, can we make it even more powerful? To this end, we attempt to add intra-layer links to introduce height, which is a 2D-3D transformation for a network, as shown in Figure \ref{fig:2D3D}. In this transformation, while width decreases, depth remains intact. It is worth mentioning that as long as the width is greater than 1, the 2D-3D transformation is feasible. 

\begin{figure}[h]
\vspace{-0.2cm}
    \centering   \includegraphics[width=0.8\linewidth]{ 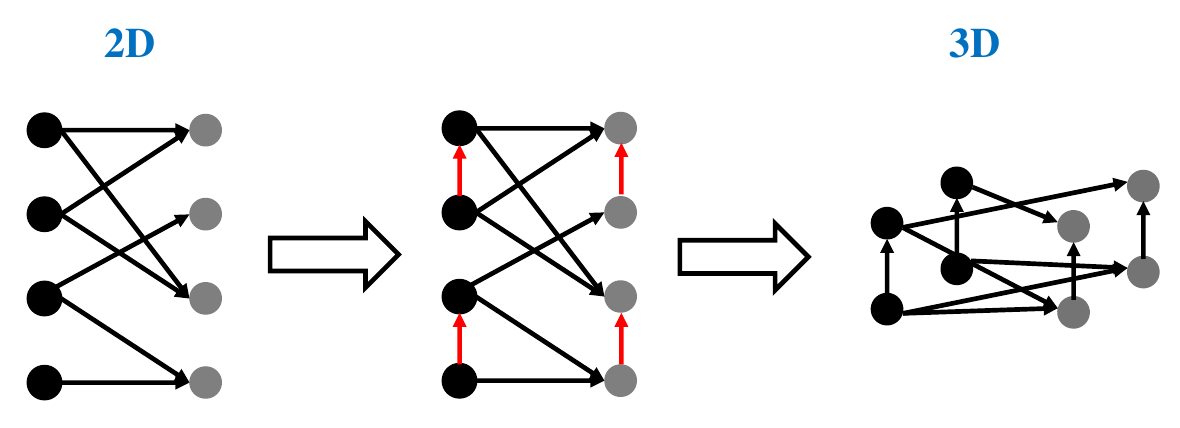}
    \caption{2D-3D transformation via the intra-layer links.}
    \label{fig:2D3D}
    \vspace{-0.3cm}
\end{figure}

\textbf{Temporal neural networks}. There are three main types of temporal neural networks: Recurrent Neural Networks (RNNs) \cite{jordan1997serial}, Long Short-Term Memory (LSTM) \cite{hochreiter1997long}, and Gated Recurrent Units (GRUs) \cite{cho2020learning}. RNNs have been a cornerstone of sequential data processing in machine learning. Introduced in the early 1980s, RNNs are designed to recognize patterns in sequences of data by maintaining a hidden state that captures information from previous time steps. However, traditional RNNs face challenges such as vanishing and exploding gradients, which hinder their ability to learn long-range dependencies effectively. To address these limitations, Hochreiter and Schmidhuber proposed LSTMs that incorporate a specialized architecture that includes memory cells and gating mechanisms, allowing them to retain information over extended periods. This design enables LSTMs to overcome the vanishing gradient problem, making them adept at capturing long-term dependencies in sequential data. LSTMs have since become the standard choice for various applications, including natural language processing, music generation, and video analysis, demonstrating superior performance in tasks that require understanding context over longer sequences. Recent advancements have further refined RNN architectures, including variations such as GRUs and attention mechanisms, which enhance the capability of sequential models to focus on relevant parts of the input data. 

There are extensive inner links between hidden neurons in temporal models, however, the proposed height network is fundamentally different from temporal models, in the sense that the proposed height network is a general-purpose model for the classification and regression of a broader range of data, though we do not exclude the possibility of translating the height network into sequential data in the future.

\textbf{Inner links in Miscellaneous Venues}. Concurrent with our research, inner connections were also explored in various studies. \cite{anderson2023connections} introduced a pairwise connection between two filters within a neural network, demonstrating that convolutional neural networks (CNNs) can benefit from such connections. Similarly, \cite{shahir2023connected} found that linking hidden neurons facilitates rapid convergence. \cite{zhang2024intrinsic} elucidated that the efficacy of spiking neural networks (SNNs) is significantly influenced by the configuration of intrinsic structures, proposing the use of inner-connection architectures to enhance learning by improving the adaptability of the integration operation. In contrast to these studies, which primarily view inner links as beneficial connections, our research examines inner links from a foundational perspective of network topology. We advocate for the consideration of these connections as a basis for height, thereby offering a more systematic and extensible framework for network design.


\section{Notation and Definition}

\noindent \textbf{Notation 1}
(Standard 2D networks). For a 2D $\mathbb{R}^{w_{0}} \rightarrow \mathbb{R}$  ReLU DNN with widths  $w_{1}, \ldots, w_{K}$ of $K$  hidden layers, we use  $\mathbf{f}_{0}=\left[f_{0}^{(1)}, \ldots, f_{0}^{\left(w_{0}\right)}\right]=\mathbf{x} \in \mathbb{R}^{w_{0}}$ to denote the input of the network. Let $\mathbf{f}_{k}=\left[f_{k}^{(1)}, \ldots, f_{k}^{\left(w_{k}\right)}\right] \in \mathbb{R}^{w_{k}}, k=1,\cdots,K,$ be the vector composed of outputs of all neurons in the $k$-th layer. The pre-activation of the $j$-th neuron in the $k$-th layer and the corresponding neuron is given by 
$$
g_{k}^{(j)}=\left\langle\mathbf{a}_{k}^{(j)}, \mathbf{f}_{k-1}\right\rangle+b_{k}^{(j)}\quad \text{and} \quad f_{k}^{(j)}=\sigma\left(g_{k}^{(j)}\right),
$$
respectively, where $\sigma(\cdot)$ is the ReLU activation, and $\mathbf{a}_{k}^{(j)} \in \mathbb{R}^{w_{k-1}}, b_{k}^{(j)} \in \mathbb{R}$ are parameters. The output of this network is $g_{K+1}=\left\langle\mathbf{a}_{K}, \mathbf{f}_{K}\right\rangle+b_{k}$ for some $\mathbf{a}_{K} \in \mathbb{R}^{w_{K}}$, $b_{K} \in \mathbb{R}$. 

\begin{definition}[Width and depth of 2D networks \cite{arora2016understanding}] 
For any number of hidden layers $K \in \mathbb{N}$, input and output dimensions  $w_{0}, w_{K+1} \in \mathbb{N}$, an $\mathbb{R}^{w_{0}} \rightarrow \mathbb{R}^{w_{K+1}}$ fully-connected network is given by specifying a sequence of $K$ natural numbers  $w_{1}, w_{2}, \ldots, w_{K}$ representing widths of the hidden layers. The depth of the network is defined as $K$, which is the number of (affine transforms, activations). The width of the network is $\max \left\{w_{1}, \ldots, w_{K}\right\}$. If and only if $w_{1}=w_{2}=\ldots=w_{K}=W$, we call such a network horizontally uniform. The width of a horizontally uniform network is $W$. Such a network is denoted as $\mathcal{N}_{W,K}^{2}$.
\end{definition}

\noindent \textbf{Notation 2} (Standard 3D networks) For a 3D $\mathbb{R}^{w_{0}} \rightarrow \mathbb{R}$ ReLU DNN with neurons $[(w_{11},..,w_{1H_1}), (w_{21},..,w_{2H_2}) \ldots, (w_{K1},..,w_{KH_K})]$ in $K$ hidden layers, where in the $k$-th layer, there are $H_k$ floors with $(w_{k1},\cdots,w_{kH_k})$ neurons at each floor, respectively, we now use the matrix $\mathbf{G}_{h_1,h_2}^k \in \mathbb{R}^{w_{kh_1}\times w_{kh_2}}$, where $h_2>h_1$, to denote the connecting operations between the $h_1$-th and $h_2$-th floors within the $k$-th hidden layer. If $\mathbf{G}^k_{h_1,h_2}(n_1,n_2)\neq 0$, it means that the output of the $n_1$-th neuron in the $h_1$-th floor is fed into the $n_2$-th neuron in the $h_2$-th floor, and multiplied by a coefficient $\mathbf{G}^k_{h_1,h_2}(n_1,n_2)$; otherwise, the output of the $n_1$-th neuron is not. $\mathbf{G}^k_{h_1,h_2 \leq h_1}=0$ by default, since no loops are allowed in a network. Similar to the classical ReLU DNN,  we use  $\tilde{\mathbf{f}}_{0}=\mathbf{x} \in \mathbb{R}^{w_{0}}$ and $\tilde{\mathbf{f}}_{k}=\left[\tilde{\mathbf{f}}_{k}^{(1)}, \ldots, \tilde{\mathbf{f}}_{k}^{\left(w_{k1}\right)},\ldots, \tilde{\mathbf{f}}_{k}^{(\sum_{t=1}^{H_k-1}w_{kt}+1)}, \ldots, \tilde{\mathbf{f}}_{k}^{\left(\sum_{t=1}^{H_k}w_{kt}\right)}\right] \in \mathbb{R}^{\sum_{t=1}^{H_k}w_{kt}}$ to denote the input and the outputs of the $k$-th layer, respectively. The $j$-th pre-activation in the $k$-th layer and the output of the network are computed as the following:  
$$g_{k}^{(j)}=\left\langle\mathbf{a}_{i}^{(j)}, \tilde{\mathbf{f}}_{k-1}\right\rangle+b_{i}^{(j)}\quad$$ and
$$\quad \tilde{f}_{k}^{(j)}=\sigma\left(g_{k}^{(j)}+\sum_{p<j}\mathbf{G}_{h(p),h(j)}^k(n(p),n(j))\tilde{f}_{k}^{(p)}\right)$$  
for each $j$, where $h(l)$ is a function mapping the $l$-th neuron into their floors, and $n(l)$ is a function mapping the $l$-th neuron into the order in its floor. For example, if $w_{k1}<l<w_{k1}+w_{k2}$, then the $(l-w_{k1})$-th neuron is in the $2$-nd floor.

\begin{definition}[Width, depth, and height of 3D networks] For an $\mathbb{R}^{w_{0}} \rightarrow \mathbb{R}$ ReLU DNN with neurons $[(w_{11},..,w_{1H_1}), (w_{21},..,w_{2H_1}) \ldots, (w_{K1},..,w_{KH_K})]$ in $K$ hidden layers, respectively, the depth of a network is $K+1$, the width is $\max_k\max\{w_{k1},\ldots, w_{kH_k}\}$, and the height is $\max\{H_1,\ldots,H_K\}$. If and only if for each layer $H_1=H_2\ldots=H_k=H$, we call such a 3D network vertically uniform. If and only if for each layer and each floor, $w_{k1}=w_{k2}=\cdots=w_{kH_k}=W$, we call such a 3D network horizontally uniform. The width, depth, and height of a horizontally and vertically uniform network are $W$, $K$, and $H$. We denote it as $\mathcal{N}_{W,K,H}^3$.   
\end{definition}

\begin{figure}[h]
\vspace{-0.2cm}
\center{\includegraphics[width=0.8\linewidth] {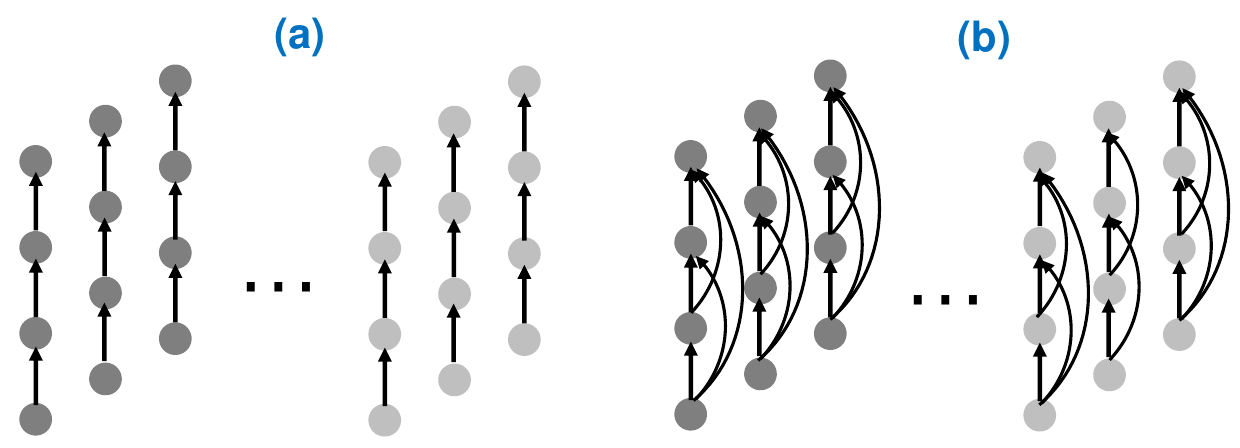}}
\caption{Two types of horizontally and vertically uniform networks are used in this paper.}
\vspace{-0.3cm}
\label{Figure_three_networks}
\end{figure}

In this paper, to avoid the confusion of notations and capture the essence of the problem, we will often assume a 2D network is horizontally uniform, and a 3D network is horizontally and vertically uniform. We highlight that we do not need to insert dense connections between two floors.
Especially, as Figure \ref{Figure_three_networks} shows, we are interested in two cases: (a) in each layer, every 2 neurons from two neighboring floors are linked, which is the simplest 3D network. In this situation, $\mathbf{G}_{h-1,h}^k\neq 0$ is a negative identity matrix for a horizontally uniform network; (b) in each layer, a neuron in the $h$-th floor is connected by a neuron from the preceding floors. In this situation, $\mathbf{G}_{<h,h}^k\neq 0$ is also a negative identity matrix for a horizontally uniform network. 

\begin{definition}[2D-3D Transformation, $\mathcal{N}_{W\times H, K}^2 \to \mathcal{N}_{W,K,H}^3$]
As Figure \ref{fig:2D3D} shows, to realize the 2D-3D transformation, \textit{i.e.}, $\mathcal{N}_{W\times H, K}^2 \to \mathcal{N}_{W,K,H}^3$, one simply intra-links every $H$ neurons in a layer in the head-to-tail manner. 
\end{definition}

Let us analyze the computational complexity in 2D-3D transformation. First, it is straightforward to see that using intra-layer links increases a few or no parameters. The reason is that one parameter is coupled with one link. Since links are sparse, the increased parameters are few. What's more, we can use the constant $-1$ or $1$ as the weight of a link, which does not increase the parameter at all. Second, the complexity of computing a layer with $HW$ neurons in a classical 2D ReLU DNN $\mathcal{N}_{HW,K}$ is $H^2W^2$ multiplications and $H^2W^2$ additions while computing a 3D ReLU DNN $\mathcal{N}_{W,K,H}$ needs $H^2W^2$ multiplications and $H^2W^2+(H-1)\cdot[W]\approx H^2W^2+WH$ additions, where $[\cdot]$ is a ceiling function, which is still quadratic. Thus, the computational cost incurred by adding intra-links is minor. When applying intra-layer links in CNNs, the links can be added between different channels. The computational cost is also minor. In brief, with sparse intra-links, 3D networks are not subjected to a high computational and parametric cost and a low training speed. Therefore, 3D networks are an economical model. If a 3D network can have a good gain compared to its original 2D network, conducting 2D-3D transformation is worthwhile.

Our focus is the network using ReLU activation and the estimation of the number of pieces. Now, we define the sawtooth function and breakpoints which are widely used in investigating the approximation ability of a network, and we will also use it to present the height separation.

\noindent \textbf{Notation 3} (sawtooth functions and breakpoints) We say a piecewise linear (PWL) function  $g: [a, b] \to \mathbb{R}$ is of “$N$-sawtooth" shape, if
$
g(x)=(-1)^{n-1}\left(x-(n-1) \cdot \frac{b-a}{N}\right),$
 for $x \in\left[(n-1) \cdot \frac{b-a}{N}, n \cdot \frac{b-a}{N}\right],n\in [N]$. We say $x_{0}\in\mathbb{R}$ is a breakpoint of a PWL function $g$, if  the left- and right-hands derivatives of $g$ at $x_{0}$ are not equal.
\vspace{-0.1cm}


\section{Approximation Mechanism of height}

Height generated by intra-layer links is a special network structure that embraces a new spatial dimension. An important question is besides the geometrical difference, does height have an essential difference from depth in approximation behaviors? Before the formal mathematical analysis and theoretical derivation, here we share our insights on this question from aspects of the basic mechanism of generating new pieces and the number of affine transforms. 

\begin{figure}[h]
\vspace{-0.2cm}
    \centering   \includegraphics[width=\linewidth]{ 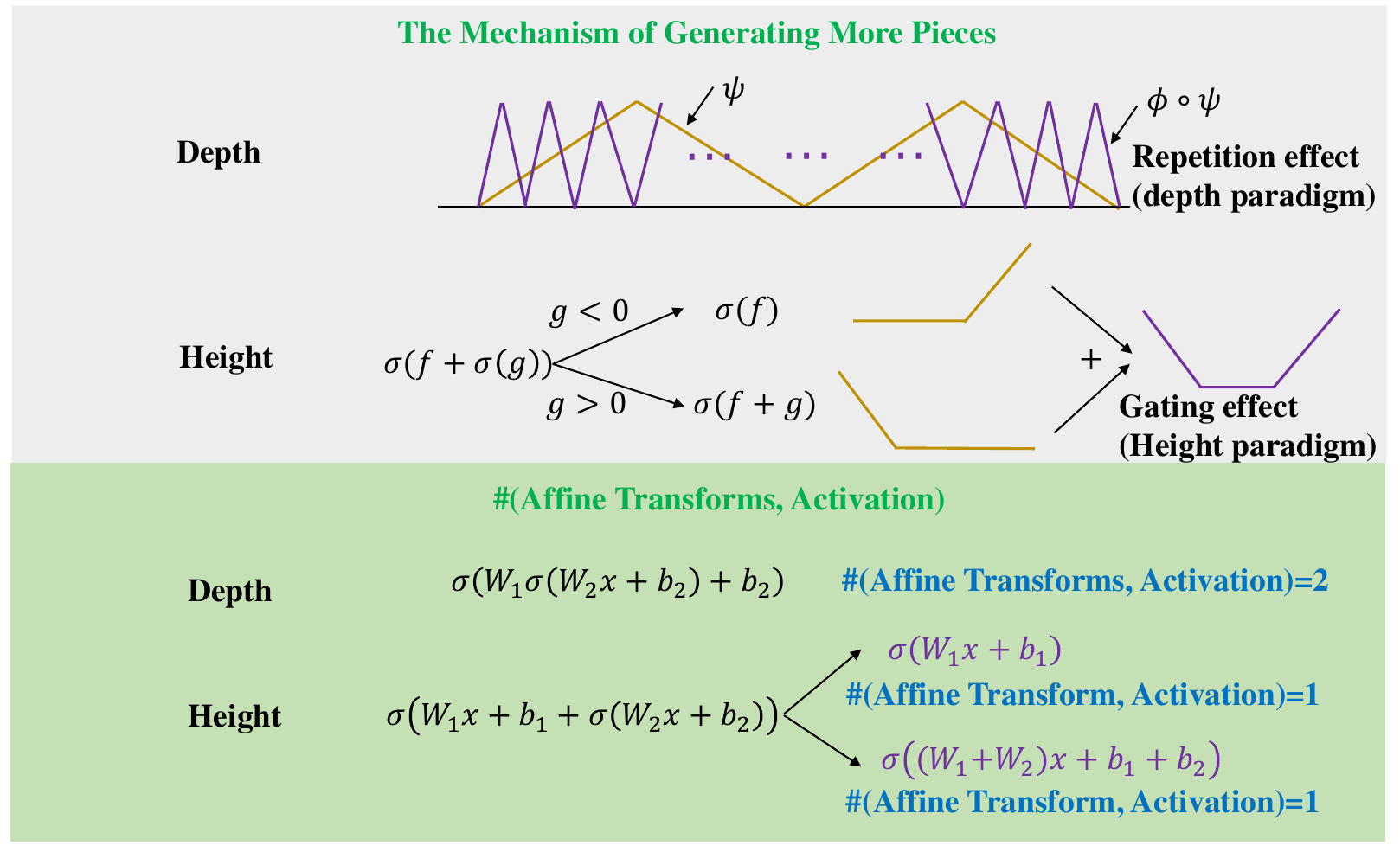}
    \caption{Differences of height and depth in accomplishing higher approximation power in terms of the mechanism of generating more pieces, the number of (affine transforms, activation), and function classes.}
    \label{fig:Mechanism}
    \vspace{-0.4cm}
\end{figure}

$\bullet$ As Figure \ref{fig:Mechanism} shows, their mechanisms of producing pieces are fundamentally different. While the mechanism of adding a new layer is the repetition effect (multiplication), \textit{i.e.}, when composing two layers that generate oscillation, each oscillation can generate more oscillations, which falls into the depth paradigm. The mechanism of height is the gating effect (addition). The neuron being embedded has two activation states, and each state is leveraged to produce a breakpoint. Two states are combined to generate more pieces. Intra-linking more neurons into a neuron can also exponentially boost the number of generated pieces.

$\bullet$ Adding height does not increase the number of affine transforms and activations. As Figure \ref{fig:Mechanism} illustrates, a fully-connected network with two layers involves two times of affine transformation and activation. In contrast, adding height actually exerts a gating effect. When $\sigma(W_2x+b_2)>0$, the output is $\sigma((W_1+W_2)x+b_1+b_2)$; when $\sigma(W_2x+b_2)=0$, the output is $\sigma(W_1x+b_1)$. The number of (affine transform, activation) is still one for both cases. We think the essence of depth is composition, which will lead to increased affine transforms and activations. Therefore, increasing height is different from increasing depth.

$\bullet$ Adding height does not increase parameters substantially. A fully-connected network of width $W$ and depth $K$ has about $W^2K$ parameters. Introducing height to a such a network increase at most $(1+2+\ldots+W-1)K\leq \frac{W^2}{2}K$ parameters, whose number of parameters is no more than a fully-connected network of width $W$ and depth $\lceil \frac{3K}{2} \rceil$. However, topologically a fully-connected network of $WK$ layers has $W^3K$ parameters.

\section{Height Can Greatly Increase the Number of Pieces}

Here, our primary argument is that height can boost the expressive power of a network in terms of the number of pieces. Our investigation consists of two parts: bound estimation and explicit construction. While the bound estimation offers initial evidence, to convincingly illustrate that 3D networks can increase the number of pieces, we need to supply the explicit construction for 3D networks to show that these bounds are tight. The number of pieces in the construction should be bigger than the maximum a 2D network can achieve. For a fair comparison, the 2D and 3D networks being compared share the same number of neurons and parameters, \textit{i.e.}, they respectively have the topologies $\mathcal{N}_{W,K,H}^3$ and $\mathcal{N}_{W\times H, K}^2$.

\subsection{Upper Bound Estimation}
\label{sec:numberofpieces}

\begin{lemma}
\label{old_new}
Let $g: \mathbb{R} \rightarrow \mathbb{R}$  be a PWL function with $w+1$ pieces, then the breakpoints of $f:= \sigma(g)$  consist of two parts: some old breakpoints of $g$ and at most $w+1$ newly produced breakpoints. Furthermore, $f$ has $w+1$ new breakpoints if and only if $g$ has $w+1$ distinct zero points.
\end{lemma}

\begin{proof}
A direct calculus.
\end{proof}

\begin{theorem}[Upper bound of 2D networks $\mathcal{N}_{W,K}^2$]
Let  $f:\mathbb{R}\rightarrow\mathbb{R}$ be a PWL function represented by an $\mathbb{R} \rightarrow \mathbb{R}$ ReLU fully-connected 2D neural network, whose depth is $K$ and widths are $w_{1}, \ldots, w_{K}$, respectively. Then $f$ has at most $\prod_{k=1}^{K}\left(w_{k}+1\right)$ pieces. Let the width and depth of a 2D network be $(W,K)$, \textit{i.e.}, $w_1=w_2=\cdots=w_k=W$, this upper bound is simplified as $\mathcal{O}(W^K)$.
\label{prop:ub_fn}
\end{theorem}

\begin{proof}
Recursively applying Lemma 3 can derive the bound. 
\end{proof}

\textbf{Remark 1.}  Theorem \ref{prop:ub_fn} is actually the univariate case of the bound: $\prod_{k=1}^{K}\sum_{j=0}^n \binom{w_{k}}{j}$, derived in \cite{montufar2017notes} for $n$-dimensional inputs. Theorem \ref{prop:ub_fn} also sharpens the bound in \cite{arora2016understanding}. Previously, they computed the number of pieces produced by a network of depth $k+1$ and widths $w_{1}, \ldots, w_{k}$ as $2^{k+1}\cdot (w_1 + 1) w_2 \cdots w_k$. The reason why their bound has an exponential term is that when considering how ReLU activation increases the number of pieces, they repetitively computed the old breakpoints generated in the previous layer. Our Lemma \ref{old_new} implies that the ReLU activation in fact cannot double the number of pieces of a PWL function.

\begin{lemma}[A corollary of Lemma \ref{old_new}]
\label{linked-struct}
    Let  $g_{1}, g_{2}: \mathbb{R} \rightarrow \mathbb{R}$  be two PWL functions with $w$  breakpoints in total.  $f_{1}:=\sigma\left(g_{1}\right)$  and  $f_{2}:=\sigma\left(g_{2}-f_{1}\right)$. Then the breakpoints of  $f_{2}$ include three parts: some breakpoints of  $g_{2}$, some breakpoints of  $f_{1}$, and at most  $2w+2$  newly produced breakpoints. Furthermore,  $f_{2}$ has $2w+2$ new breakpoints if and only if  $g_{2}-   f_{1}$  has  $2w+2$  distinct zero points.    
\end{lemma}

Let us briefly explain why 3D architectures can produce more pieces. Given two PWL functions $g_{1}$ and $g_{2}$ which have a total of $w$ breakpoints, in the 2D architecture,  $\sigma\left(g_{1}\right)$  and  $\sigma\left(g_{2}\right)$  have totally at most  $2w+2$  breakpoints, which contains at most $w$ old breakpoints of  $g_{1}, g_{2}$  and at most  $2w+2$  newly produced breakpoints. However, in the 3D architecture, $\sigma\left(g_{2}-\sigma\left(g_{1}\right)\right)$, corresponding to height=2 by taking $g_1$ and $g_2$ as the pre-activation of neurons of two floors, can produce more breakpoints because $\sigma(g_1)$ has two states: activated or deactivated. Then, $\sigma(g_1)$ and $\sigma\left(g_{2}-\sigma\left(g_{1}\right)\right)$ consist of at most $w$ old breakpoints of $g_{1}, g_{2}$ and $(w+1)+(2w+2)=3w+3$ new breakpoints.

\begin{theorem}[Upper bound of three dimensional networks $\mathcal{N}_{W,K,H}^3$]
\label{linked_bound}
Let  $f: \mathbb{R} \rightarrow \mathbb{R}$  be a PWL function represented by a ReLU DNN with depth $K$, widths $[(w_{11},..,w_{1H_1}), (w_{21},..,w_{2H_2}) \ldots, (w_{K1},..,w_{KH_K})]$, where $w_{k1}=w_{k2}=\ldots=w_{kH_k}=w_k$, and height $H_k$ in each layer, and the topology of this network is Figure \ref{Figure_three_networks}(a) that $\mathbf{G}_{h-1,h}^{k}$ is a negative identity matrix, then $f$ has at most  $\prod_{k=1}^{K}\left((2^{H_k}-1)w_{k}+1\right)$ pieces. Assume this network is vertically and horizontally uniform, \textit{i.e.}, $w_1=w_2=\cdots=w_K=W$ and $H_1=H_2=\cdots=H_K=H$, this upper bound is simplified as $\mathcal{O}((2^H-1)W+1)^K)$.
\end{theorem}

\begin{proof} 
For conciseness, we only consider the case of $H=2$, and $w_{k1}=w_{k2}=w_k$ for the $k$-th layer, $1\leq k \leq K$. The general case is nothing but repeating the same analysis in each layer several times. We prove by induction on $K$. For the base case  $k=1$, we assume for $j>1$, the neurons  $\tilde{f}_{1}^{(2j-1)}$  and  $\tilde{f}_{2}^{(2j)}$ are linked, where the $2j-1$-th neuron is in the first floor and $2j$-th neuron is in the second floor. The number of breakpoints of  $\tilde{f}_{1}^{(2j)}$, $j=1, \ldots, w_{1}$, is at most $2+(-1)^{j}$. Hence, the first layer yields at most  $3w_{1}+1$ pieces. For the induction step, we assume that for some  $K \geq 1$, any  $\mathbb{R} \rightarrow \mathbb{R}$  ReLU DNN with every two neurons linked in each hidden layer, depth  $K$  and widths  $w_{1}, \ldots, w_{K}$  of  $K$  hidden layers produces at most  $\prod_{k=1}^{K}\left(w_{k}+1\right)$  pieces. Now we consider any $\mathbb{R} \rightarrow \mathbb{R}$  ReLU DNN with every two neurons linked in each hidden layer, depth  $K+1$  and widths  $w_{1}, \ldots, w_{K+1}$ of $K+1$ hidden layers. By the induction hypothesis, each  $\tilde{g}_{K+1}^{(2j-1)}$  has at most  $\prod_{k=1}^{K}\left(3 w_{k}+1\right)-1$  breakpoints. Then the breakpoints of  $\sigma(\tilde{g}_{K+1}^{(2j-1)})$  consist of some breakpoints of $\tilde{g}_{K+1}^{(2j-1)}$ and at most  $\prod_{k=1}^{K}\left(3 w_{k}+1\right)$  newly generated breakpoints. Then  $\tilde{g}_{K+1}^{(2j)}-\tilde{f}_{K+1}^{(2j-1)}$  has at most  $2 \cdot \prod_{k=1}^{K}\left(3 w_{k}+1\right)-1$  breakpoints, based on Lemma \ref{linked-struct}. The breakpoints of  $\tilde{f}_{K+1}^{(2j)}=\sigma(\tilde{g}_{K+1}^{(2j)}-\tilde{f}_{K+1}^{(2j-1)})$  consist of some breakpoints of  $\tilde{g}_{K+1}^{(2j)}-\tilde{f}_{K+1}^{(2j-1)}$ and at most $2\cdot  \prod_{k=1}^{K}\left(3 w_{k}+1\right)$  newly generated breakpoints. Note that $\tilde{g}_{K+1}^{(1)}, \ldots, \tilde{g}_{K+1}^{(2w_{K+1}-1)}$  have totally at most $\prod_{k=1}^{K}\left(3 w_{k}+1\right)-1$ breakpoints. In all, the number of pieces we can therefore get is at most
$
1+w_{K+1} \cdot\left(\prod_{k=1}^{K}\left(3 w_{k}+1\right)+2 \cdot \prod_{k=1}^{K}\left(3 w_{k}+1\right)\right)+\prod_{k=1}^{K}\left(3w_{k}+1\right)-1=\prod_{k=1}^{K+1}\left(3w_{k}+1\right).
$
\vspace{-0.1cm}
\end{proof}

Comparing Theorem \ref{linked_bound} with Theorem \ref{prop:ub_fn}, it is found that given the same number of neurons $W\times K \times H$ and the same number of parameters, arranging them into a 3D network $\mathcal{N}_{W,K,H}^3$ can result in $\mathcal{O}((2^H-1)W+1)^K)$ pieces, which is exponentially larger than $\mathcal{O}(HW+1)^K)$ obtained from a 2D network $\mathcal{N}_{W\times H, K}^2$.

In Supplementary Materials, we supply the bound estimation for multivariate 2D and 3D networks, which shows that a 3D network can achieve a much higher number of linear regions.

\subsection{Tightness of Bounds.} 

We offer constructions to show that this bound is achievable in a depth-bounded but width-unbounded network (depth=3) (Proposition \ref{tight_bound_1}) and a width-bounded (width=3) but depth-unbounded network (Proposition \ref{construct_4^k}) in one-dimensional space. Previously, many bounds \cite{pascanu2013number, montufar2014number,  montufar2017notes, xiong2020number} on linear regions were derived, however, it is unknown whether these bounds are vacuous or tight, particularly for networks with more than one hidden layer. Determining the tightness of a bound is essential in analyzing the approximation ability of a deep network. What makes Propositions \ref{tight_bound_1} and \ref{construct_4^k} special is that they for the first time substantiate that \cite{montufar2017notes}'s bound is tight over an arbitrary three-layer network and deeper networks with small widths, which fills the gap of bound estimation.

The constructions for 3D networks with height 2 in Propositions ~\ref{tight_bound_2} and ~\ref{achieve 3w/2} have a number of pieces larger than the upper bounds of 2D networks. In Proposition~ \ref{sqrt7}, by enumerating all possible cases, we present a construction for a 3D network $\mathcal{N}_{1,K,2}$ which has the width 1, height 2, and arbitrary depth whose number of pieces is larger than $\mathcal{N}_{2,K}$ which has the width 2 and arbitrary depth possibly achieves. Proposition \ref{theorem_n_linked_1_layer} shows that $\left(\frac{(H+1)H}{2}+1\right)$ pieces can be achieved by a 3D network $\mathcal{N}_{1,1,H}$. Proposition \ref{theorem_n_linked_lb_1} provides rather tight constructions for the 3D network $\mathcal{N}_{1,K,4}$. The constructions in Propositions \ref{sqrt7}, \ref{theorem_n_linked_1_layer}, and \ref{theorem_n_linked_lb_1} also have a larger number of pieces than the upper bounds of 2D networks. Therefore, our constructions rigorously confirm that a 3D network has a more powerful expressive ability than a 2D network. 

\begin{proposition}[The bound $\prod_{k=1}^{K}\left(w_{k}+1\right)$ is tight for a depth-bounded but width-unbounded 2D network]
\label{tight_bound_1}
Given an $\mathbb{R} \rightarrow \mathbb{R}$ two-hidden-layer ReLU network, for any width $w_{1} \geq 3, w_{2} \geq 2$ in the first and second hidden layers, there exists a PWL function represented by such a network, whose number of pieces is  $\left(w_{1}+1\right)\left(w_{2}+1\right)$.
\label{twoproduct}
\end{proposition}

\begin{proof}
\begin{figure}[h]
    \centering
    \includegraphics[width=\linewidth]{ 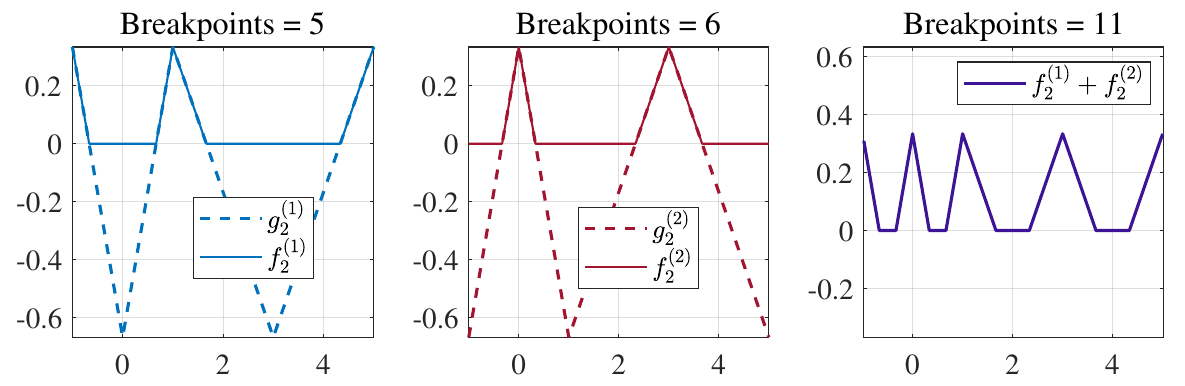}
    \caption{Construction of PWL functions to reach the bound of Proposition \ref{tight_bound_1} when $w_{1}=3$, $w_{2}=2$.}
    \label{fig:construction_feedforwar}
\end{figure}

To guarantee the bound $\prod_{k=1}^{K}\left(w_{k}+1\right)$ is tight, the following two requirements should be met: (i) each  $g_{k}^{(j)}$,  $k=0,1,2$, $j=1, \ldots, w_{k}$, has distinct zero points that are as many as its number of pieces, so that the activation step can produce the most new breakpoints; (ii) the breakpoints of each  $g_{(k+1)}^{(j)}$,  $k=0,1,2$, $j=1, \ldots, w_{k+1}$, as the affine combination of  $\left\{f_{k}^{(1)}, \ldots, f_{k}^{\left(w_{k}\right)}\right\}$, contains all the breakpoints of  $\left\{g_{k}^{(1)}, \ldots, g_{k}^{\left(w_{k}\right)}\right\}$, so that all the old breakpoints are reserved.

Now we give the proof in detail. Let $f_{1}^{(1)}(x)=\sigma(3 x),
f_{1}^{(2)}(x)=\sigma(-x+3),
f_{1}^{(3)}(x)=\sigma\left(\frac{3}{2} x-\frac{3}{2}\right)
$. When $w_{1}=3$, we set

$$\begin{array}{l}
g_{2}^{(1)}=-\left(f_{1}^{(1)}+f_{1}^{(2)}-f_{1}^{(3)}-3-\frac{1}{w_{2}+1}\right), \\
g_{2}^{(j)}=f_{1}^{(1)}+f_{1}^{(2)}-f_{1}^{(3)}-3-\frac{j}{w_{2}+1}, j=2, \ldots, w_{2} .
\end{array}$$

When $w_{1}>3$, we let  $f_{1}^{(j)}=\sigma(-2 x-2(j-3))$  and for $j=2, \ldots, w_{2}$,

\begin{equation}
\begin{aligned}
    & g_{2}^{(1)} \\
    = & -f_{1}^{(1)}-f_{1}^{(2)}+f_{1}^{(3)}-\sum_{j=4}^{w_{1}}(-1)^{j-1} f_{1}^{(j)}+3+\frac{1}{w_{2}+1},
\end{aligned}
\end{equation}

\begin{equation}
\begin{aligned}
& g_{2}^{(j)} \\
= & f_{1}^{(1)}+f_{1}^{(2)}-f_{1}^{(3)}+\sum_{r=4}^{w_{1}}(-1)^{r-1} f_{1}^{(r)}-3-\frac{j}{w_{2}+1}.
\end{aligned}
\end{equation}

Then $g_{2}^{(j)}$  has  $w_{1}+1$ distinct zero points. Hence for $j=1, \ldots, w_{2}$, the breakpoints of  $f_{2}^{(j)}=\sigma\left(g_{2}^{(j)}\right)$ keeps all breakpoints of  $g_{2}^{(j)}$  and  yields $w_{1}+1$ new breakpoints. Note that $f_{2}^{(j)}$  and  $f_{2}^{(j)}$  do not share new breakpoints, and $f_{2}^{(1)}$  and  $f_{2}^{(2)}$  covers all the breakpoints of  $\left\{g_{2}^{(j)}\right\}_{j=1}^{w_{2}}$. Therefore, the total number of pieces via an affine combination of $f_{2}^{(1)}, \ldots, f_{2}^{\left(w_{2}\right)}$ is  $\left(w_{1}+1\right)\left(w_{2}+1\right)$ pieces.
\end{proof}

\begin{proposition}[The bound $\prod_{k=1}^{K}\left(w_{k}+1\right)$ is tight for a width-bounded but depth-unbounded 2D network]
\label{construct_4^k}
    Given an $\mathbb{R} \rightarrow \mathbb{R}$ ReLU network with width $w$ for the first layer and 3 for other layers, for any depth $K\geq 2$, there exists a PWL function represented by such a network, whose number of pieces is  $(w+1)\cdot 4^{K-1}$.
\end{proposition}

\begin{proof}
Let  $f_{1}^{(1)}, \ldots f_{1}^{(w)}$  be the same as in Proposition \ref{tight_bound_1}. Let
$$
\tilde{g}_{2}=\left\{\begin{array}{c}
f_{1}^{(1)}+f_{1}^{(2)}-f_{1}^{(3)}-3, \quad \quad \quad \quad \quad ~~~\text { if } w=3, \\
f_{1}^{(1)}+f_{1}^{(2)}-f_{1}^{(3)}+\sum_{j=4}^{w}(-1)^{j-1} f_{1}^{(j)}-3, ~ \text{if}~ w>3 .
\end{array}\right.
$$
We set
$$
\begin{array}{l}
f_{2}^{(1)}=\sigma\left(2 \tilde{g}_{2}-\frac{1}{3}\right), \\
f_{2}^{(2)}=\sigma\left(-\tilde{g}_{2}+\frac{2}{3}\right), \\
f_{2}^{(3)}=\sigma\left(\frac{3}{2} \tilde{g}_{2}-\frac{1}{2}\right) .
\end{array}
$$
Now we continue our proof by induction. Assume that we have constructed $f_{k}^{(1)}$, $f_{k}^{(2)}$  and $ f_{k}^{(3)}$, $k \geq 2$, then we set
$$
\tilde{g}_{k+1}=f_{k}^{(1)}+f_{k}^{(2)}-f_{k}^{(3)}-\frac{3}{6^{k}}
$$
and
$$
\begin{array}{l}
f_{k+1}^{(1)}=\sigma\left(2 \tilde{g}_{k+1}-\frac{2}{6^{k}}\right), \\
f_{k+1}^{(2)}=\sigma\left(-\tilde{g}_{k+1}-\frac{4}{6^{k}}\right), \\
f_{k+1}^{(3)}=\sigma\left(\tilde{g}_{k+1}+\frac{3}{6^{k}}\right) .
\end{array}
$$
Through a direct calculus, we know $\tilde{g}_{k+1}$ has $(w+1) \cdot 4^{k-1}$ pieces with opposite slopes in every two adjoint pieces and ranges from $0$ to  $3 / 6^{k}$  in each piece except the leftmost and rightmost pieces, which implies we can obtain a total of $(w+1) \cdot 4^{K-1}$ pieces.
\end{proof}

\begin{proposition}[The bound $\left((2^H-1)W+1\right)^K$ is tight for $\mathcal{N}_{W,2,2}$]
\label{tight_bound_2}
Given an $\mathbb{R} \rightarrow \mathbb{R}$ two-hidden-layer vertically uniform ReLU network, $\mathcal{N}_{W,2,2}$, for any even $W \geq 2$, there exists a PWL function represented by such a network, whose number of pieces is $\left(3 W+1\right)^2$.
\end{proposition}

\begin{proof}
Since $H=2$, based on our setting, all odd neurons lie on the first floor of a layer, while all even neurons lie on the second floor of a layer. 
To guarantee the bound $\prod_{k=1}^{K}\left(3w_{k}+1\right)$ is tight, the following two conditions should be satisfied: (i) $\tilde{g}_{k}^{(2j-1)}$ and  $\tilde{g}_{k}^{(2j)}-\tilde{f}_{k}^{(2j-1)}$ have as many zero points as possible so that $\sigma(\tilde{g}_{k}^{(2j-1)})$ and $\sigma(\tilde{g}_{k}^{(2j)}-\tilde{f}_{k}^{(2j-1)})$ can produce the maximal number of breakpoints; (ii) 
all old breakpoints of $\left\{\tilde{g}_{k}^{(1)}, \ldots, \tilde{g}_{k}^{\left(2w_{k}-1\right)}\right\}$ are reserved by $\tilde{g}_{k+1}^{(2j-1)}$, an affine transform of $\left\{\tilde{f}_{k}^{(1)}, \ldots, \tilde{f}_{k}^{\left(2w_{k}-1\right)}\right\}$. 

We first consider the first hidden layer. Let
$$
\begin{array}{l}
\tilde{f}_{1}^{(1)}(x)=\sigma\left(\frac{9}{2} x-27\right),
~\tilde{f}_{1}^{(2)}(x)=\sigma\left(\frac{3}{2} x-\tilde{f}_{1}^{(1)}(x)\right) \\
\tilde{f}_{1}^{(3)}(x)=\sigma(-2 x+2),
~~\tilde{f}_{1}^{(4)}(x)=\sigma\left(-x+2-\tilde{f}_{1}^{(3)}(x)\right) \\
\tilde{f}_{1}^{(5)}(x)=\sigma\left(-\frac{7}{2} x-\frac{7}{4}\right),
~\tilde{f}_{1}^{(6)}(x)=\sigma\left(-2 x+8-\tilde{f}_{1}^{(5)}(x)\right).
\end{array}
$$
When  $w_{1}=6$, we set $\tilde{g}=-\frac{2}{9} \tilde{f}_{1}^{(1)}-\tilde{f}_{1}^{(2)}+\frac{1}{2} \tilde{f}_{1}^{(3)}+\tilde{f}_{1}^{(4)}-\frac{4}{7} \tilde{f}_{1}^{(5)}-\tilde{f}_{1}^{(6)} .$

When $w_{1}>6$, for $j\geq 4$,  let
$
\tilde{f}_{1}^{(2j-1)}=\sigma\left(-5\left(x-a_{2j-1}+3\right)\right),
\tilde{f}_{1}^{(2j)}=\sigma\left(-2\left(x-a_{2j-1}\right)-\tilde{f}_{1}^{(2j-1)}\right),
$
where $a_{2j-1}=-\frac{19}{2}-9\left(\frac{2(j-1)}{2}-3\right)$, then the output of the first layer is expressed as the following:
$\tilde{g}=-\frac{2}{9} \tilde{f}_{1}^{(1)}-\tilde{f}_{1}^{(2)}+\frac{1}{2} \tilde{f}_{1}^{(3)}+\tilde{f}_{1}^{(4)}-\frac{4}{7} \tilde{f}_{1}^{(5)}-\tilde{f}_{1}^{(6)}+\sum_{j=4}^{w_{2}}(-1)^{\frac{2j-1}{2}}\left(\frac{2}{5} f_{1}^{(2j-1)}+f_{1}^{(2j)}\right),
$
which has $\frac{3}{2} w_{1}+1$ pieces and whose adjacent pieces have slopes of opposite signs. Note that any line $y=b$, where $b \in (-13 / 2, -6)$, can cross all pieces of $\tilde{g}+b$. Thus, $g$ fulfills the conditions of Lemma \ref{linked-struct}. We divide the breakpoints of $\tilde{g}$ into two parts: $B_{upper}=\left\{ x:x \text{ is a breakpoint of } \tilde{g} \text{ and } \tilde{g}(x)>b \right\}$ and $B_{lower}=\left\{ x:x \text{ is a breakpoint of } \tilde{g} \text{ and } \tilde{g}(x)\leq b \right\}$. We refer to their counts as \#$B_{upper}$ and \#$B_{lower}$.

\begin{figure}[h]
\vspace{-0.2cm}
    \centering
    \includegraphics[width=\linewidth]{ 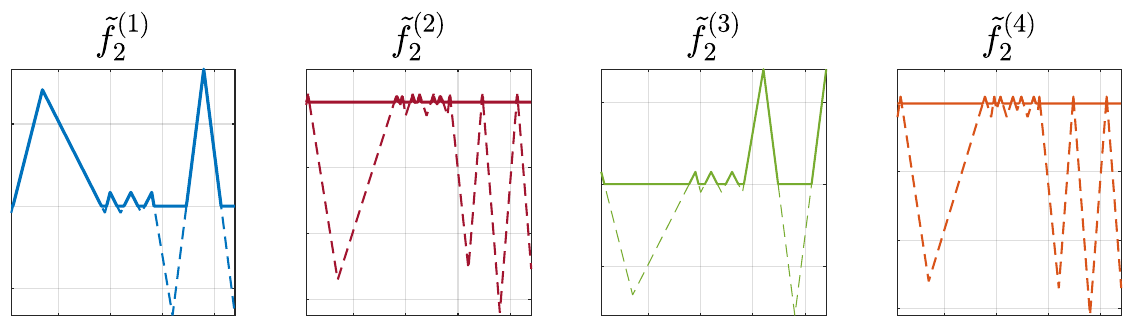}
    \caption{The PWL functions that reach the bound of Proposition \ref{tight_bound_2} when $w_{1}=6$, $w_{2}=4$.}
    \label{fig:construction_inta}
    \vspace{-0.3cm}
\end{figure}

Next, we construct the second hidden layer. $\tilde{f}_{2}^{(1)}:=\sigma\left(\tilde{g}+b_1\right)$, where $b_1 \in (-13/2,-6)$, has $3 w_{1}+ 
1$ new breakpoints. Then by choosing some scaling parameter  $a \in(0,1)$ bias $b_{2}$ to fulfill Lemma \ref{linked-struct}, we can also make   $a\tilde{g}+b_{2}-\tilde{f}_{2}^{(1)}$ has $3 w_{1}+2$ distinct zero-points, which implies  $\tilde{f}_{2}^{(2)}:=\sigma\left(a \tilde{g}+b_{2}-\tilde{f}_{2}^{(1)}\right)$ has  $3 w_{1}+2$  newly produced breakpoints. Therefore, the affine combination of $\tilde{f}_{2}^{(1)}$ and  $\tilde{f}_{2}^{(2)}$ contains all breakpoints of $B_{upper}$, and has  \#$B_{upper}+\left(3 w_{1}+1\right)+\left(3 w_{1}+2\right)$ breakpoints. To reserve all the breakpoints of $\tilde{g}$, we do the similar thing for $-\tilde{g}$ to gain $\tilde{f}^{(3)}_{2}$ and $\tilde{f}^{(4)}_{2}$, whose affine combination has \#$ B_{lower}+\left(3 w_{1}+1\right)+2\cdot\left(3 w_{1}+2\right)$ breakpoints, which contains all breakpoints in $B_{lower}$, and shares no breakpoints with the affine combination of $\left\{\tilde{f}^{(1)}_{2},\tilde{f}^{(2)}_{2} \right\} $.  

Hence, the affine combination of $\left\{\tilde{f}^{(1)}_{2},\tilde{f}^{(2)}_{2},\tilde{f}^{(3)}_{2},\tilde{f}^{(4)}_{2} \right\} $ has  
\#$B_{upper}+$ \# $B_{lower}+ 2\cdot\left ( 3w_{1}+1\right )+2\cdot\left (6w_{1}+2\right )=\left ( 3w_{1} \right )+ 4\cdot\left ( 3w_{1}+1\right )$
breaking points, which contains all the breakpoints of $\tilde{g}$. $\left\{\tilde{f}^{(1)}_{2},\tilde{f}^{(2)}_{2},\tilde{f}^{(3)}_{2},\tilde{f}^{(4)}_{2} \right\}$ are visualized in Figure \ref{fig:construction_inta}. Repeating this procedure by selecting different $b_1, a, b_2$, we can construct the remaining $\{\tilde{f}_{2}^{(i)}\}_{i=5}^{w_2}$
 such that the affine transformation of $\{\tilde{f}_{2}^{(i)}\}_{i=1}^{w_2}$ has pieces of $3w_{1}+3w_{2} \cdot\left(3w_{1}+1\right)+1=\left(3 w_{1}+1\right)\left(3 w_{2}+1\right)$.
\end{proof}

\begin{proposition}[Use 3D networks to achieve $\prod_{k=1}^{K}\left(3w_{k}\right)$ pieces]
\label{achieve 3w/2}
There exists a $[0,1] \rightarrow \mathbb{R}$ function represented by a vertically uniform 3D ReLU DNN with depth $K$ and height $2$, and the same width in each floor of a layer, $w_{1}, \ldots, w_{K}$ respectively, whose number of pieces is at least $3w_{1} \cdot \ldots \cdot 3 w_{K}$.
\end{proposition}

\begin{proof}
Let $\phi(x)=x$ defined over $[0, \Delta]$. The core of the proof is to use a one-hidden-layer network of $w\geq 2$ neurons and height=2 to create $3w$ pieces from $\phi(x)$. 

\begin{figure}[h]
    \centering
    \includegraphics[width=\linewidth]{ 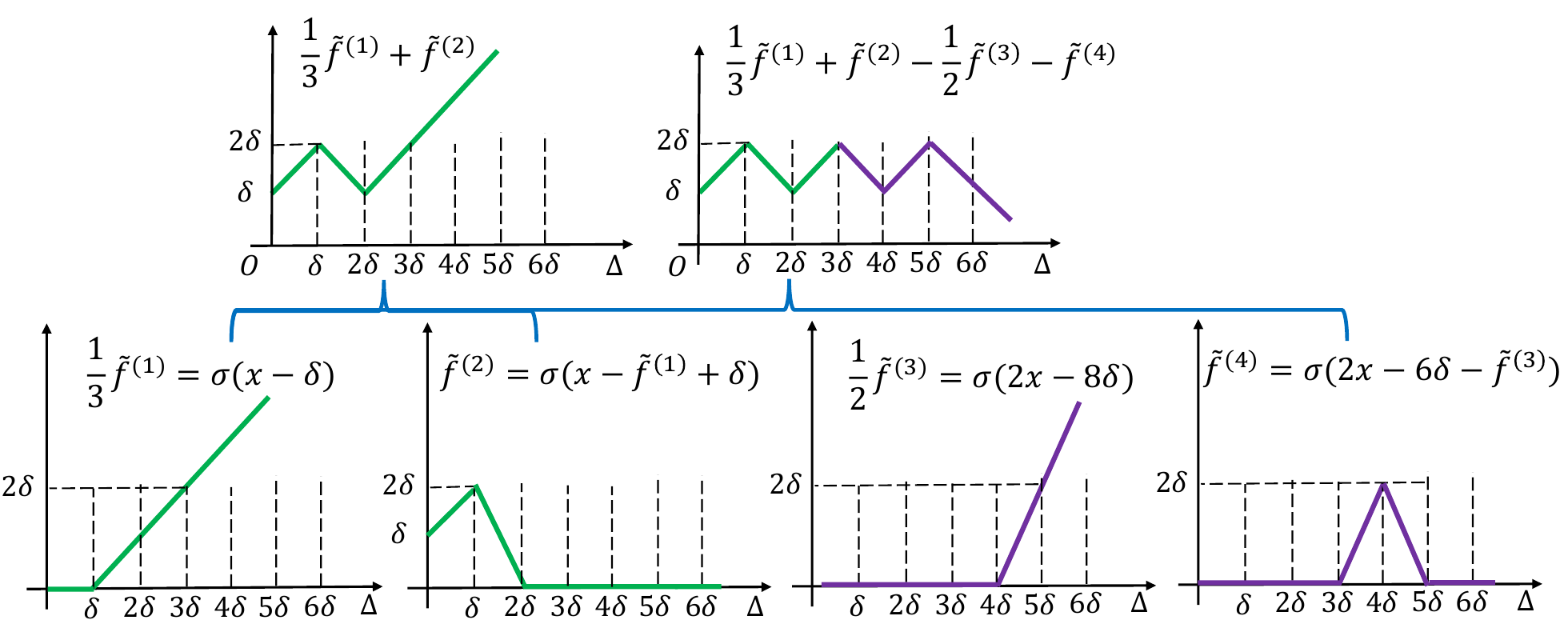}
    \vspace{-0.5cm}
    \caption{A schematic illustration of how to use an intra-linked network to generate a sawtooth function.}
    \label{fig:sawtoth}
    \vspace{-0.3cm}
\end{figure}

Let $\delta=\frac{2 \Delta}{3 w}$. Set $\tilde{g}^{(1)}=3 \phi-3 \delta$,  $\tilde{f}^{(1)}=\sigma\left(\tilde{g}^{(1)}\right)$,  $\tilde{g}^{(2)}=\phi$,  $\tilde{f}^{(2)}=\sigma\left(\tilde{g}^{(2)}-\tilde{f}^{(1)}+\delta\right)$, and $\tilde{g}^{(2 j+1)}=4\phi-4(3j+1)\delta$, $f^{(2 j+1)}=\sigma\left(\tilde{g}^{(2 j+1)}\right)$, $\tilde{g}^{(2 j+2)}=2\phi-6 j \delta$, $\tilde{f}^{(2 j+2)}=\sigma\left(\tilde{g}^{(2 j+2)}-\tilde{f}^{(2 j+1)}\right)$, for all  $j=1,\ldots, w-1 $. The output of this one-hidden-layer network is
$
\xi_{\Delta, w}(x)=\frac{1}{3} \tilde{f}^{(1)}+\tilde{f}^{(2)}-\delta+\sum_{j=1}^{w-1}(-1)^{j}\left(\frac{1}{2} \tilde{f}^{(2 j+1)}+\tilde{f}^{(2 j+2)}\right),
$
which has $3w$ pieces on  $[0, \Delta]$. $\xi_{\Delta, w}(x)$ is of slope  $(-1)^{j}$  on  $[j \delta,(j+1) \delta]$, $j=0, \ldots, 3 w/2-1$, and ranges from $0$ to  $\delta$ on each piece. Figure \ref{fig:sawtoth} shows how the affine transformation of $\{\tilde{f}^{(1)},\tilde{f}^{(2)},\tilde{f}^{(3)},\tilde{f}^{(4)} \}$ constructs a sawtooth function of 6 pieces.
Please note that flipping $\phi(x)$ or translating $\phi(x)$ will not prevent $\xi_{\Delta, w}(\phi(x))$ from generating $3w$ pieces.

The targeted intra-linked ReLU network with depth  $K$, height $2$  and width  $w_{1}, \ldots, w_{K}$ of $K$ hidden layers is designed as $\xi_{\Delta_K, w_K}\circ \xi_{\Delta_{K-1}, w_{K-1}}\circ\cdots \circ \xi_{\Delta_1, w_1}(x)$,
where $\Delta_k = 1/\Big(\prod_{k=1}^{K-1} 3 w_{k}\Big)$.
\end{proof}

\begin{proposition}[3D networks $\mathcal{N}_{1,K,2}^3$ can greatly increase the number of pieces compared to 2D networks with $\mathcal{N}_{2,K}^2$]
\label{sqrt7}
Let  $f: \mathbb{R} \rightarrow \mathbb{R}$  be a PWL function represented by an $\mathbb{R} \rightarrow \mathbb{R}$ $(K+1)$-layer ReLU DNN with widths $2$ of all $K$  hidden layers. Then the number of pieces of $f$ is at most
$\left\{\begin{array}{cc}
\sqrt{7}^{K}, & \text { if } K \text { is even, } \\
3 \cdot \sqrt{7}^{K-1}, & \text { if } K \text { is odd. }
\end{array}\right.
$
There exists an $\mathbb{R} \rightarrow \mathbb{R}$ $K$-hidden-layer horizontally and vertically uniform ReLU DNN of width $1$ and height $2$, which can produce at least $7 \cdot 3^{K-2}+2$ pieces.
\end{proposition}

\begin{proof}
For the first assertion, we claim that each pre-activation $g_{k}^{(j)}$, $2\leq k\leq K$, $j=1,2$, cannot make its two adjacent pieces have slopes with different signs, which implies the pre-activation cannot produce the most breakpoints as in Lemma \ref{old_new}. In fact,  $g_{2}^{(j)}$, $j=1,2$, has at most $3$ pieces. If some  $g_{2}^{(j)}$ has $3$ pieces, then by enumeration, we know either it has a $0$ slope, or it has two adjacent pieces with slopes of the same sign (see Figure \ref{fig:manifold}). Hence,  $f_{2}^{(j)}$, $j=1,2$, has at most $2$ new breakpoints. Then the output of the second layer has at most $2+2 \times 2=6$  breakpoints and $7$ pieces. Applying a similar method to each piece, we can finish the proof via a simple induction step.

Now we come to the second assertion. For convenience, we say an $\mathbb{R} \rightarrow \mathbb{R}$ PWL function $f$  is of “triangle-trapezoid-triangle" shape on $[a, b] \subset \mathbb{R}$, if there exists a partition of $[a, b]: a<x_{1}<x_{2}<\cdots<x_{6}<b$  and a positive constant  $c$, such that

$$f(x)=\left\{\begin{array}{cc}
c, & ~~~~~~~~~~\text {if } x=a, x_{2}, x_{5}, b \\
-c, & ~~\text {if } x=x_{1}, x_{6} \\
-3 c, & ~~~~\text {if } x \in\left[x_{3}, x_{4}\right] \\
\text {linear connection, } & \text {otherwise. }
\end{array}\right.$$

Given a PWL function  $f: \mathbb{R} \rightarrow \mathbb{R}$  of “triangle-trapezoid-triangle" shape on  $[a, b]$, with a partition  $a<x_{1}<x_{2}<\cdots<x_{6}<b$ and  $f(a)=c>0$, if we set
$$
\begin{array}{l}
g^{(1)}=4 f, \\
g^{(2)}=2 f-\frac{3 c}{2}, \\
f^{(1)}=\sigma\left(g^{(1)}\right), \\
f^{(2)}=\sigma\left(g^{(2)}-f^{(1)}\right),
\end{array}
$$
then  $g=-\frac{1}{4} f^{(1)}+f^{(2)}+\frac{c}{8}$  is of “triangle-trapezoid-triangle" shape on  $\left[a, x_{2}\right],\left[x_{2}, x_{5}\right]$,  and  $\left[x_{5}, b\right]$, respectively.

Using this fact, we can construct a  PWL  function represented by a  $K$-layer 1-neuron-wide, and 2-neuron-tall 3D ReLU DNN, which has  $7 \cdot 3^{K-2}+2$  pieces. Actually, if we set
$$
\begin{array}{l}
\tilde{f}_{1}^{(1)}=\sigma(2 x), \\
\tilde{f}_{1}^{(2)}=\sigma\left(x-\tilde{f}_{1}^{(1)}+1\right), \\
\tilde{g}_{2}^{(1)}=-4 \tilde{f}_{1}^{(2)}+2, \\
\tilde{g}_{2}^{(2)}=-2 \tilde{f}_{1}^{(2)}+\frac{3}{2},
\end{array}
$$
then through a direct calculus,  $\frac{1}{4} \tilde{f}_{2}^{(1)}+\tilde{f}_{2}^{(2)}-\frac{3}{8}$  is of “triangle-trapezoid-triangle" shape on  $[-1,1]$. Using the fact above repeatedly, we can construct a  PWL  function represented by an  $\mathbb{R} \rightarrow \mathbb{R}$ $K$-layer, $1$-wide, 2-neuron-tall 3D ReLU DNN, which is constant on  $(-\infty,-1] \cup   [1, \infty)$  and of “triangle-trapezoid-triangle" shape on  $\left[-1+\frac{2 n}{3^{K-2}},-1+\frac{2(n+1)}{3^{K-2}}\right]$, $n=0, \ldots, 3^{K-2}-1$. 

\begin{figure}[h]
    \centering
\includegraphics[width=0.6\linewidth]{ 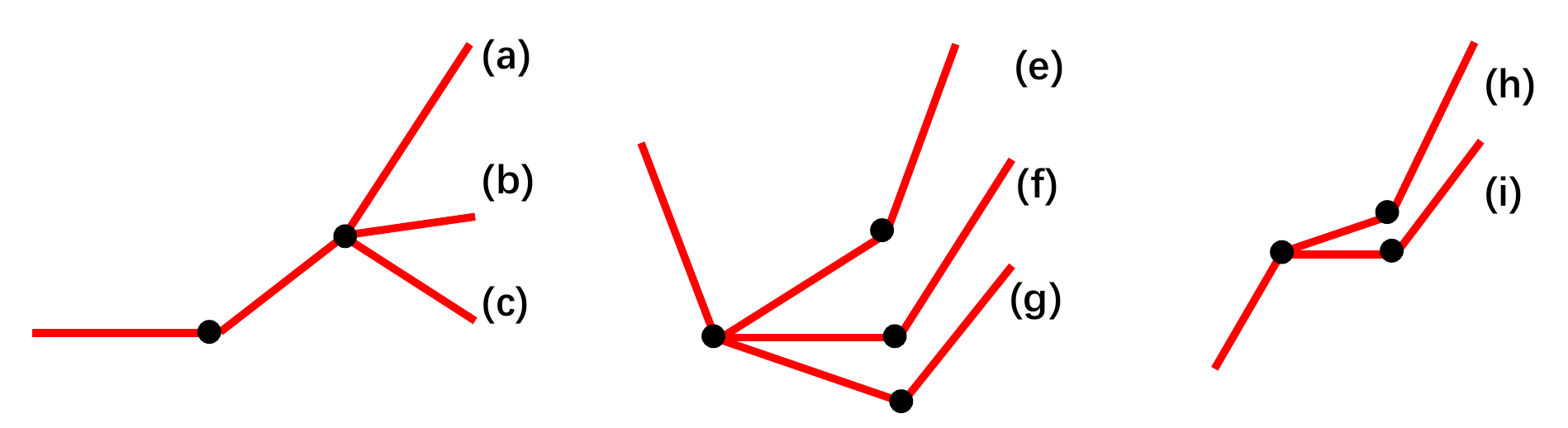}
    \caption{Enumerating all possible shapes of $g_{2}^{(j)}$,$j=1,2$ in Proposition \ref{sqrt7}.}
    \label{fig:manifold}
\end{figure}
\end{proof}

\begin{proposition}[ $\left(\frac{(H+1)H}{2}+1\right)$ pieces for a network $\mathcal{N}_{1,1,H}^3$]
Given an $\mathbb{R} \rightarrow \mathbb{R}$ 3D ReLU network, $\mathcal{N}_{1,1,H}$, there exists a PWL function represented by such a network, whose number of pieces is $\frac{(H+1)H}{2}+1$.
\label{theorem_n_linked_1_layer}
\end{proposition}

\begin{proof}
Without loss of generality, a one-hidden-layer, one-neuron-wide, and $H$-neuron-tall 3D network $\mathcal{N}_{1,1,H}$ is mathematically formulated as the following:
\begin{equation}
    \begin{cases}
        & \tilde{f}^{(1)} = \sigma(w^{(1)} x+b^{(1)})  \\
        & \tilde{f}^{(j+1)} = \sigma(w^{(j)} x+b^{(j)}-\tilde{f}^{(j)}) \\
    \end{cases}.
\end{equation}

\begin{figure}[h]
    \centering
    \includegraphics[width=\linewidth]{ 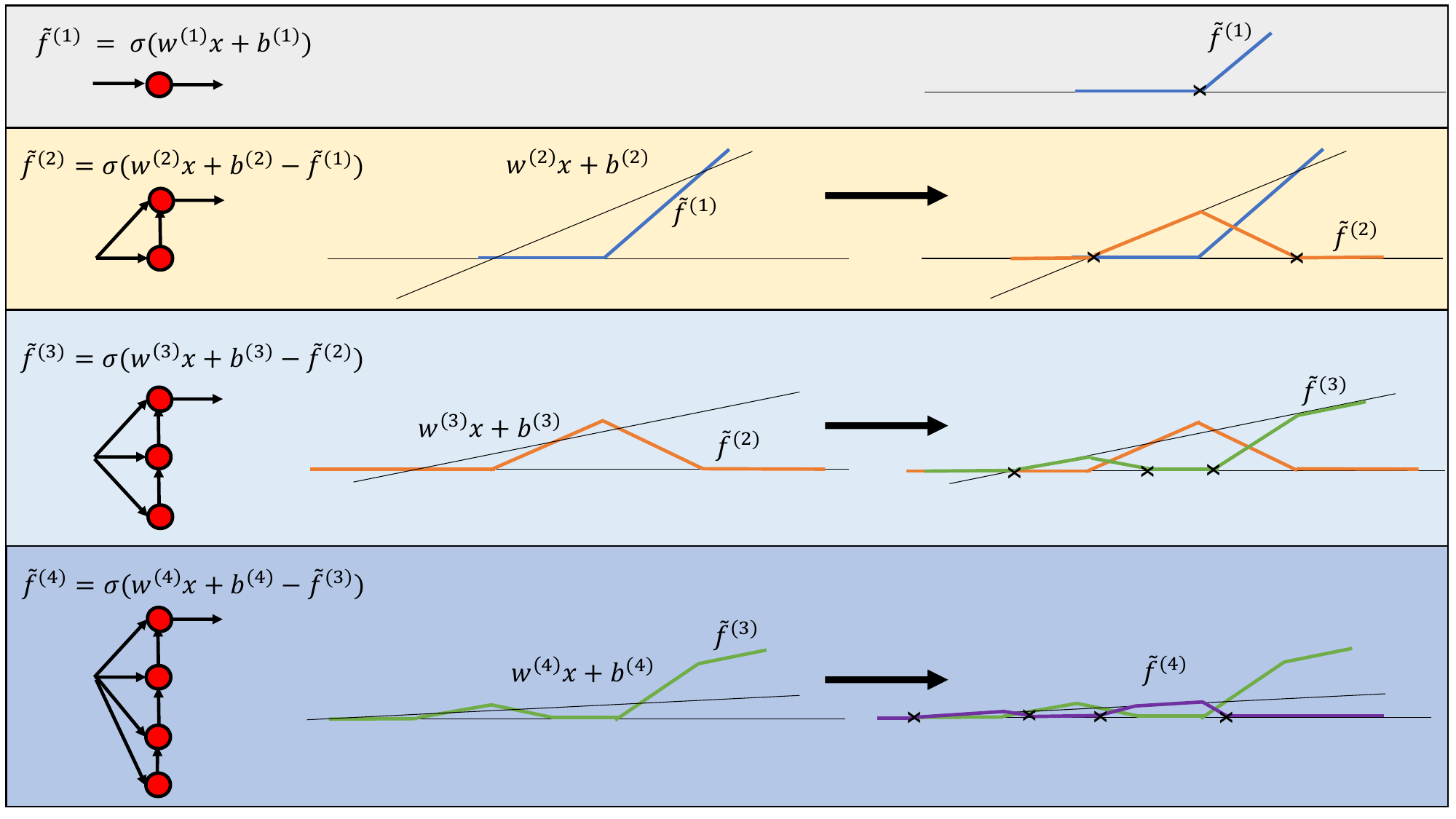}
    \caption{The construction demonstrating that the bound $\prod_{h=1}^{H}\left(\frac{H(H+1)}{2}+1\right)$ is tight for $\mathcal{N}_{1,1,H}$.}
\label{fig:n_links_construction}
    \vspace{-0.3cm}
\end{figure}

To prove that the bound $\prod_{h=1}^{H}\left(\frac{(H+1)H}{2}+1\right)$ is tight for a one-hidden-layer network, the key is to make each $\tilde{f}^{(j)}$ produce $j$ new breakpoints and have $j$ non-zero pieces that share a point with $y=0$. 
We use mathematical induction to derive our construction. Figure \ref{fig:n_links_construction} schematically illustrates the key idea in our construction.

First, let $\tilde{f}^{(1)}=\sigma(x+1)$ and $\tilde{f}^{(2)}=\sigma(0.5\times(x+2)-\tilde{f}^{(1)})$. Note that $\tilde{f}^{(1)}$ has $1$ non-zero piece that shares a point with $y=0$, and $\tilde{f}^{(2)}$ has $2$ non-zero pieces that share a common point with $y=0$. 

Then, given $\tilde{f}^{(j)}, j\geq 2$, we suppose $\tilde{f}^{(j)}$ has $j$ non-zero pieces that share a point with $y=0$. Since $\tilde{f}^{(j)}$ is continuous, we select its peaks $\{(x_{p_i}, \tilde{f}^{(j)}(x_{p_i}))\}$ by the following conditions: i) $\tilde{f}^{(j)}$ is not differentiable at $x_{p_i}$; ii) $\tilde{f}^{(j)}(x_{p_i}) \neq 0$. Next, let $(x^*, \tilde{f}^{(j)}(x^*))$ be the lowest peak of $\tilde{f}^{(j)}$. As long as the slope $w^{(j+1)}$ and the bias $b^{(j+1)}$ satisfy
\begin{equation}
\begin{cases}
        & w^{(j+1)} <\frac{\tilde{f}_{1}^{(j)}(x^*)}{x^*+j+1} \\
        & b^{(j+1)} = w_{j+1}\times (j+1)
\end{cases},
\end{equation}
$w^{(j+1)}x+b^{(j+1)}$ crosses and only crosses $j$ pieces of $\tilde{f}^{(j)}$. These pieces are exactly non-zero pieces that share a point with $y=0$. Thus, plus the breakpoint $-\frac{b^{(j+1)}}{w^{(j+1)}}$, $\tilde{f}^{(j+1)}$ generates a total of $j+1$ new breakpoints. At the same time, $\tilde{f}^{(j+1)}$ has $j+1$ non-zero pieces that share a point with $y=0$. Figure \ref{fig:n_links_construction} illustrates the process of induction.

Finally, the total number of breakpoints is $\sum_{j=1}^{H} j=\frac{(H+1)H}{2}$, which concludes our proof.

\end{proof}


\begin{proposition}[An arbitrarily deep 3D network $\mathcal{N}_{1,K,4}$ of width=1, height=4 and each floor of a layer intra-linked based on $\mathbf{G}_{h-1,h}^k\neq 0$ for $h=2,3,4$, can achieve at least $9^K$ pieces]
There exists an $\mathbb{R} \rightarrow \mathbb{R}$ function represented by a horizontally and vertically uniform 3D ReLU DNN $\mathcal{N}_{1,K,4}$ with width $1$, height $4$, and each floor of a layer intra-linked based on $\mathbf{G}_{h-1,h}^k\neq 0$ for $h=2,3,4$, whose number of pieces is at least $9^K$.
\label{theorem_n_linked_lb_1}
\end{proposition}

\begin{proof}
The core of the proof is to use a one-hidden-layer, one-neuron-wide, and $4$-neuron-tall 3D network $\mathcal{N}_{1,1,4}$ to create a quasi-sawtooth function with as many pieces as possible. We construct four neurons as follows:
\begin{equation}
    \begin{cases}
        & \tilde{f}^{(1)} = \sigma(2x)  \\
        & \tilde{f}^{(2)} = \sigma(x+1-\sigma(\tilde{f}^{(1)})) \\
        & \tilde{f}^{(3)} = \sigma(\frac{1}{3}(x+2)-\tilde{f}^{(2)}) \\
        & \tilde{f}^{(4)} = \sigma(\frac{1}{9}(x+3)-\tilde{f}^{(3)})
    \end{cases}.
\end{equation}
The profiles of $\tilde{f}^{(1)},\tilde{f}^{(2)},\tilde{f}^{(3)},\tilde{f}^{(4)}$ are shown in Figure \ref{fig:n_quasi_sawtoth}(a).  

\begin{figure}[h]
    \centering
    \includegraphics[width=\linewidth]{ 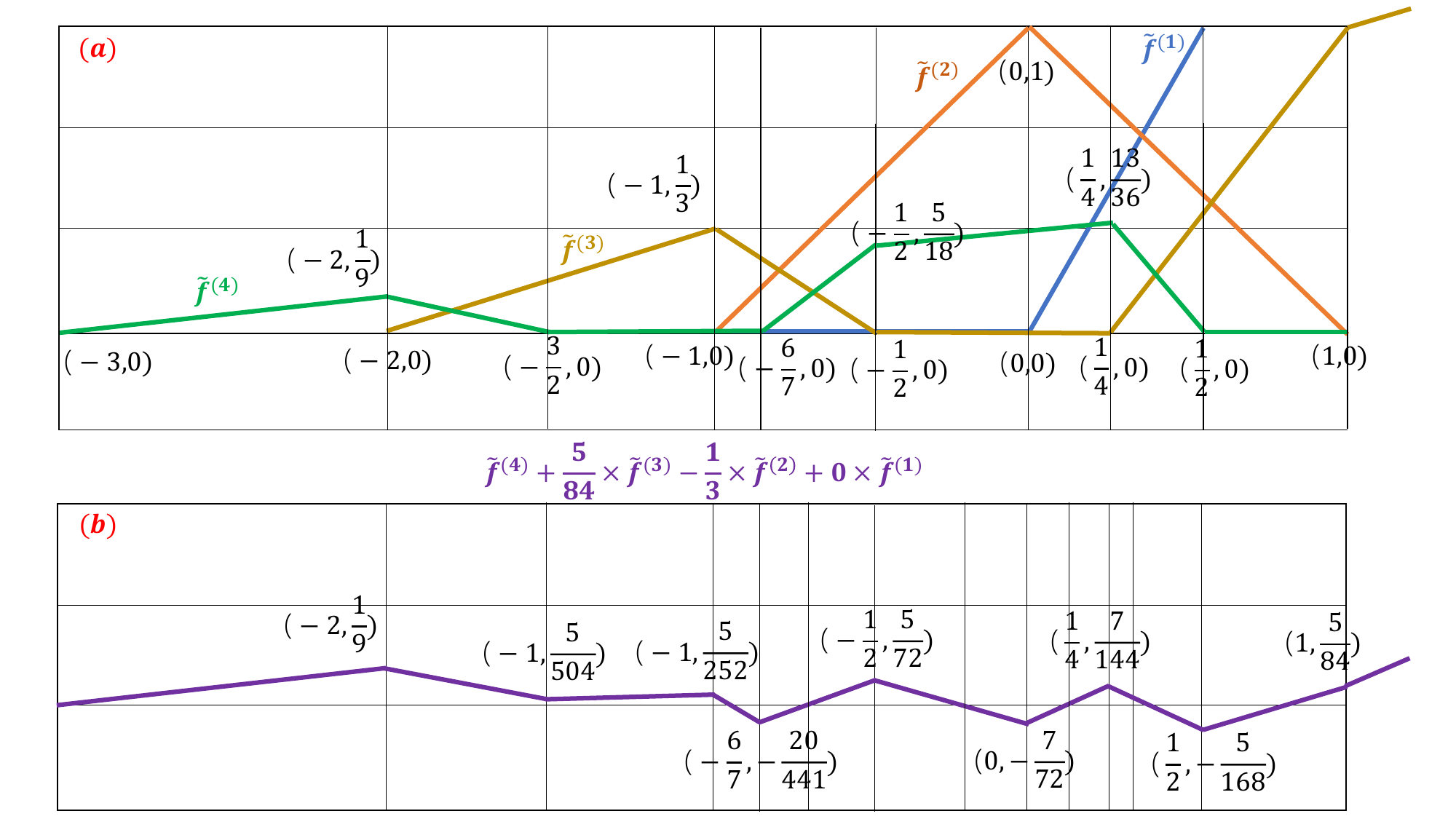}
    \caption{A schematic illustration of how to use an intra-linked network to generate a sawtooth function.}
    \label{fig:n_quasi_sawtoth}
    \vspace{-0.3cm}
\end{figure}

By combining $\tilde{f}^{(1)},\tilde{f}^{(2)},\tilde{f}^{(3)},\tilde{f}^{(4)}$ with carefully calibrated coefficients, we have the following quasi-sawtooth function that has $9$ pieces are 
\begin{equation}
    \eta(x)=\tilde{f}^{(4)}+\frac{5}{84}\times \tilde{f}^{(3)}-\frac{1}{3}\times \tilde{f}^{(2)}+0\times  \tilde{f}^{(1)}.
\end{equation}
As shown in Figure \ref{fig:n_quasi_sawtoth}(b), we have marked all breakpoints of $\eta(x)$ to validate its correctness.

Next, we just need to let each layer of the intra-linked network represent a stretched and down-pulled variant of $\eta(x)$, \textit{e.g.}, the $k$-th layer $\eta_k(x)=M_k\cdot\eta(x)-B_k$, where $M_k$ is a sufficiently large number and $B_k>\frac{5}{504}M_k+3$ to ensure that $[-3,0]$ is within the function range of $\eta_k(x)$.

Finally, the constructed network is 
\begin{equation}
    \eta_{K}\circ \eta_{K-1}\circ\cdots \circ \eta_{1}(x).
\end{equation}

\end{proof}

\section{Height Can Greatly Improve The Approximation Rate}

The above compares the number of pieces generated by 2D and 3D networks, which serves as side evidence for the stronger approximation power of height. 
Here, we would like to compare the universal approximation error for a general class of functions, \textit{i.e.}, polynomials, which is the main-stream measurement of the approximation power of a network \cite{yarotsky2017error,liang2016deep}. In the error bounds of existing results, only the depth appears in the exponent part while the width is at the base. In this work, \textbf{by introducing intra-layer links into networks, we construct ReLU FNNs with width $\mathcal{O}(W)$ and depth $\mathcal{O}(K)$ can approximate polynomials in $[0,1]^d$ with error $\mathcal{O}\left(2^{-2WK}\right)$, which is a non-trivial extension of the result $\mathcal{O}\left(W^{-K}\right)$ in \cite{shen2019deep} and $\mathcal{O}\left(2^{-K}\right)$ in \cite{yarotsky2017error} for fixed width networks}. From the perspective of height, such a network is of width $\mathcal{O}(W)$, depth $\mathcal{O}(K)$, and height $\mathcal{O}(W)$. This also means that without increasing the number of parameters, 2D-3D transformation can bring great improvement.

Here, we construct a 3D ReLU network to approximate an arbitrary polynomial. The main steps are summarized below:
\begin{itemize}
    \item \textbf{Construction of sawtooth functions}: We construct ReLU FNNs to implement a series of sawtooth functions.
    \item \textbf{Approximation of product $\times(x,y)=xy:[0,1]^2\to [0,1]$}: It suffices to approximate the squaring function $f(x)=x^2$ since $xy=\frac{(x+y)^2-x^2-y^2}{2}$, which is given by the composition and combination of sawtooth functions.
    \item \textbf{Approximation of polynomials}: We extend the bivariate product $\hat{\times}_2=\hat{\times}(x,y)\approx xy$ to multivariate case $\hat{\times}_d:[0,1]^d$ by recursively define $\hat{\times}_{k}(x_1,\ldots,x_k)=\hat{\times}\left( \hat{\times}_{k-1}(x_1,\ldots,x_{k-1}),x_k \right)$. Notably, we slightly modify the above product to $\tilde{\times}_d$  for all $\boldsymbol{x}\in Q=[-1,1]^d$ while preserving $\tilde{\times}(\boldsymbol{x})=0$ if one of $x_1,\ldots,x_d$ is zero.
\end{itemize}

\textbf{Step 1. Construction of Sawtooth Functions.} We first show how to use a 3D ReLU network to efficiently construct the sawtooth function.

\begin{theorem}\label{Sawtooth}
    Let $g_1: [0,1]\to[0,1],g_1(x)=1-2|x-\frac{1}{2}|$. We generate “sawtooth" function $g_s: [0,1]\to[0,1]$ by recursively computing $g_s=g_1\circ g_{s-1}$. Then $g_s$ can be implemented by a 3D ReLU network $\mathcal{N}_{2,1,s}$.
\end{theorem}

\begin{proof}
    When $s=1$, $g_s$ is directly given by
\begin{displaymath}
    g_s(x)=2x-4\sigma\left(x-\frac{1}{2} \right),x\in[0,1].
\end{displaymath}
Now we use an induction step assuming that the affine combination of the first $k$ neurons, $k=1,\ldots,s-1$, and $y=x$ gives $g_{k}$. Then it is sufficient to prove that the output $g_{s}$ can be given by the affine combination of first $s$ neurons and $y=x$, since the latter can be given by two neurons, \textit{i.e.}, $x=\sigma(x)+\sigma(-x)$. Using $-g_{s}+\frac{1}{2}$ as the pre-activetion of the $s$-th neuron, we have $g_{s+1}=-4\sigma\left(-g_{s}+\frac{1}{2} \right)-\frac{1}{2}g_{s}+2$, which intra-link neurons in the same layer and then affinely combines $s$ neurons in the subsequent layer and $y=x$ by the induction hypothesis. 
\end{proof}

In addition to being a fundamental construction in approximation, the sawtooth function also helps to represent periodic functions. For example, for any $x\in [-1,1]$, we have $\cos{\left(\pi 2^s x\right)}=\cos{\left(\pi g_s(|x|)\right)}$ (see \cite{Dennis2021deep}). This shows height can also improve results in periodic and trigonometric approximation.



\textbf{Step 2. Approximation of product}. We show how to approximate $y=x^2$ with composition and combination of sawtooth functions, which is the key technique of the proof. Then, we leverage the formula $xy=\frac{1}{4}(x+y)^2-\frac{1}{4}(x-y)^2$ to construct the product. 

\begin{lemma}[Approximating $y=x^2$]\label{Squaring}
Let $f_s:[0,1]\to [0,1]$ be the piecewise linear interpolation of $y=x^2$ on $x=\frac{l}{2^s},l=0,1,\ldots,2^s$. Then for any $K,W\in\mathbb{N}^+$, $f_{s=KW}$ can be implemented by a 3D ReLU network $\mathcal{N}_{2W+2,K,W}$. 
    
\end{lemma}

\begin{proof}
    Note that for any $j\geq 2$, $f_{j-1}-f_{j}=\frac{g_j}{2^{2j}}$, where $g_j$ is the sawtooth function with $j$ pieces. Thus, we have
\begin{displaymath}
    f_{s}(x)=f_{1}(x)+\sum_{j=2}^{s}\left(f_{j}-f_{j-1} \right)=x-\sum_{j=1}^{s}\frac{g_j(x)}{2^{2j}}.
\end{displaymath}
As shown in Figure \ref{Figure_SquaringNet}, $f_{KW}$ can be implemented by a 3D network $\mathcal{N}_{2W+2,K,W}$. In each layer except the first, $W$ compositions of $g_W$ and an identity operation are performed, which has a width no more than $2W+2$ and a height no more than $W$. 

\end{proof}

\begin{figure}[htb!]
\center{\includegraphics[width=\linewidth] {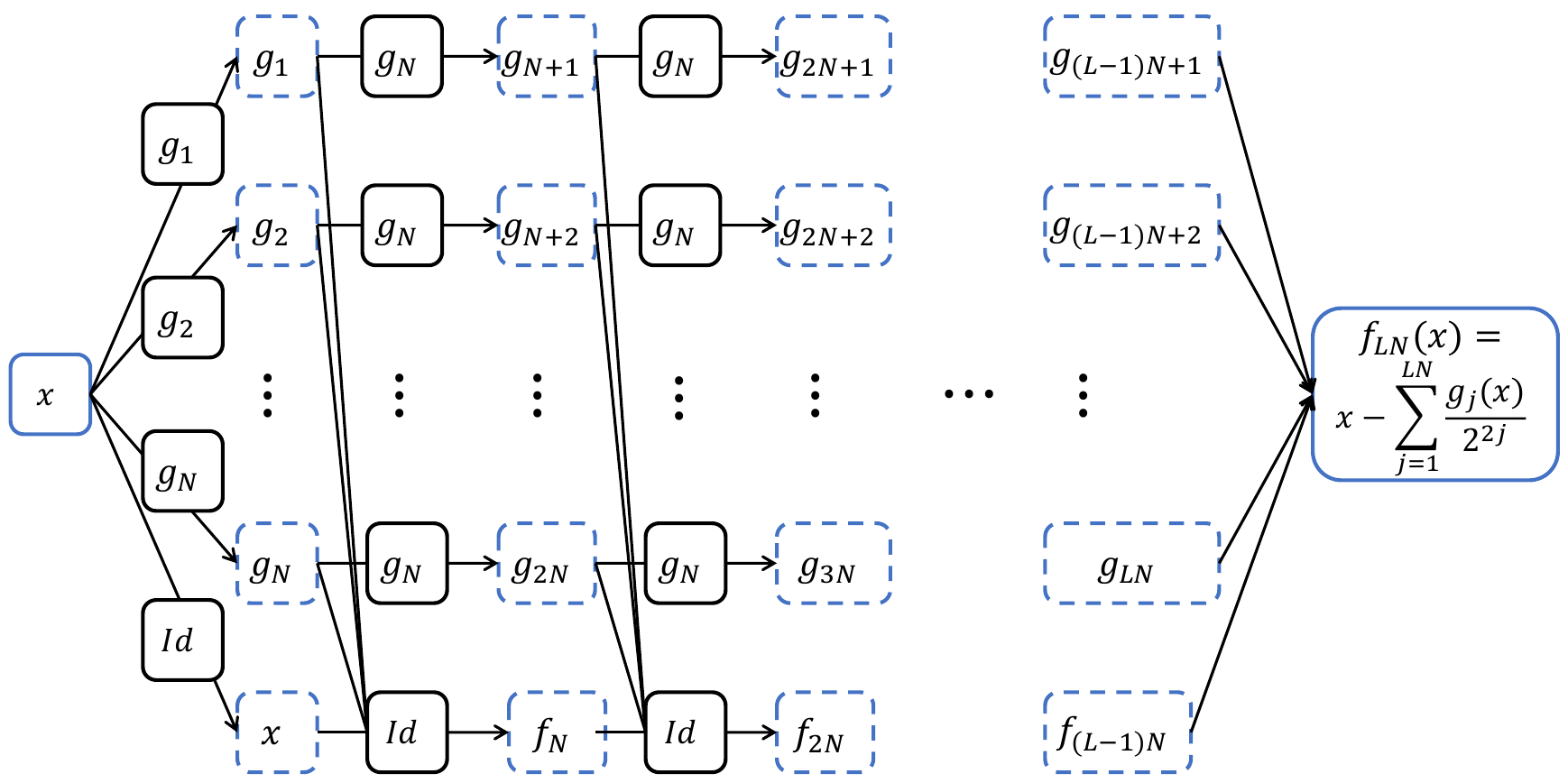}}
\caption{An illustration of a network $\mathcal{N}_{W+2,K,W}$ that implements the product function.}
\label{Figure_SquaringNet}
\end{figure}

\textbf{Remark 3.} In literature, \textit{e.g.}, \cite{lu2021deep,liang2016deep}, fully-connected networks of width $\mathcal{O}(W^2)$ and depth $\mathcal{O}(K)$ are constructed to reach approximation error of $\mathcal{O}(W^{-2K})$, \textit{i.e.}, with almost same number of neurons, we reach a much better approximation error of $\mathcal{O}(2^{-2WK})$ by adding height.

Now using the fact that $xy=\frac{(x+y)^2}{4}-\frac{(x-y)^2}{4}$, we easily give an approximation to the product function as follows.

\begin{theorem}[Approximation of the product $xy$] \label{Multiplication}
    For any $L,N\in\mathbb{N}^+$, there exist a function  $\widetilde{\times}(x,y):[0,1]^2\to[0,1]$ implemented by a 3D ReLU network $\mathcal{N}_{4(W+1),K,W}$ such that
\begin{displaymath}
    \left\|\widetilde{\times}(x,y)-xy \right\|_{L^{\infty}}\leq2^{-2KW-1}.
  \end{displaymath}
Besides, $\widetilde{\times}(x,y)=0$ if $x=0$ or $y=0$.
\end{theorem}

\begin{proof}
    Let $f_{KW}$ be a 3D network approximating the squaring function defined in Figure \ref{Squaring}. We set 
\begin{displaymath}
    \widetilde{\times}(x,y)=2f_{KW}\left( \frac{x+y}{2}\right)-2f_{KW}\left( \frac{x-y}{2}\right).
\end{displaymath}
Using the fact that for any $s\in\mathbb{N}^{+},x\in[0,1]$, $|f_s(x)-x^2|\leq 2^{-2(s+1)}$, we have

\begin{align}
    & ~~~~~~\nonumber |\widetilde{\times}(x,y)-xy| \\
    &=2 \left | f_{KW}\left( \frac{x+y}{2}\right)-f_{KW}\left( \frac{x-y}{2}\right) \right.  \\
    & ~~~-\left. \left( \frac{x+y}{2} \right)^2+\left( \frac{x-y}{2} \right)^2 \right |\\ \nonumber
    &\leq 2\left|f_{KW}\left( \frac{x+y}{2}\right)-\left( \frac{x+y}{2} \right)^2 \right| \\
    & ~~~+2\left|f_{KW}\left( \frac{x-y}{2}\right)-\left( \frac{x-y}{2} \right)^2 \right| \nonumber\\
    &\leq 2^{-2KW-1},\quad\text{for all } x,y\in [0,1]\nonumber.
\end{align}

It remains to verify $\widetilde{\times}(x,y)\in[0,1]$ for all $x,y\in[0,1]$. Observe that $\widetilde{\times}(x,y)$ is linear in $[0,1]\setminus \left( D_1\cup D_2 \cup D_3 \right)$ where 
\begin{displaymath}
    D_1=\left\{ \left( \frac{2j}{2^{KW}},y\right),j=0,1,\ldots,2^{KW-1},y\in[0,1] \right\},
\end{displaymath}

\begin{displaymath}
    D_2=\left\{\left( x, \frac{2k}{2^{KW}}\right),x\in[0,1],k=0,1,\ldots,2^{KW-1} \right\},
\end{displaymath}
and
\begin{displaymath}
    D_3=\left\{ ( x,y)\in[0,1]^2,x+y=\frac{2l}{2^{KW}},l=0,1,\ldots,2^{KW-1} \right\}.
\end{displaymath}
Besides, $\widetilde{\times}(x,y)=xy$ on $D_1\cap D_2 \cap D_3=\left\{\left( \frac{2j}{2^{KW}},\frac{2k}{2^{KW}}\right),j,k=0,1,\ldots,2^{KW-1}\right\}$. Therefore, $\widetilde{\times}(x,y)$ is the piece-wise linear interpolation of $xy$ on $D_1\cap D_2 \cap D_3$. Since the values on all the interpolation points are in $[0,1]$, their linear interpolation $\widetilde{\times}(\cdot,\cdot)$ is also in $[0,1]$.

Since each $f_{KW}$ can be implemented by a 3D ReLU network $\mathcal{N}_{2(W+1),K,W}$. Then the above product $\widetilde{\times}$ can be implemented by combining two such networks in parallel, which is of depth $K$, width $4(W+1)$, and height $W$.

\end{proof}

\textbf{Step 3. Approximation of Polynomials.} Instead of dealing with an arbitrary polynomial, here we take the product of multiple variables as a representative example. The derivation for an arbitrary polynomial can be easily generalized from the product of multiple variables. 


\begin{theorem}[Extension of product] \label{ExtendMultiplication}
For any $K,W,d\in\mathbb{N}^{+}$ with $K\geq 3,d\geq 2$, there exists a function $\widehat{\times}_{d}:[0,1]^d\to [0,1]$ implemented by a 3D ReLU network of depth $(d-1)K$ and width $4W+d+2$ such that 
	\begin{displaymath}
		\left\|\widehat{\times}_{k}(\boldsymbol{x})-x_1 x_2\cdots x_d \right\|_{L^{\infty}}\leq \frac{d-1}{2}\cdot2^{-2KW},\forall \boldsymbol{x}\in\mathbb{R}^d
	\end{displaymath}
 Besides, $\widehat{\times}(\boldsymbol{x})=0$ if one of $x_1,\ldots,x_d=0$.
\end{theorem}

\begin{proof}
	We recursively define $\widetilde{\times}_{1}(x_1)=x_1$ and 
	\begin{displaymath}
		\widetilde{\times}_{k}(x_{1},\ldots,x_k)=\widetilde{\times}\left(\widetilde{\times}_{k-1}(x_{1},\ldots,x_{k-1}),x_k\right),\;k\geq 2,
	\end{displaymath}
	where  $\widetilde{\times}:[0,1]^2\to [0,1]$ is the bivariate product function in \ref{Multiplication}, implemented by a ReLU FNN of width $4(W+1)$, depth $K$ and height $W$.
	Then following a similar procedure in the proof, we have

\begin{align}
    & ~~~~~~\nonumber \left\|\widetilde{\times}_{d}(\boldsymbol{x})-x_1\cdots x_d \right\|_{L^{\infty}} \\
    \nonumber
    &\leq \left\| \widetilde{\times}_d (x_1,\ldots,x_d)-\widetilde{\times}_{d-1} (x_1,\ldots,x_{d-1})x_d \right\|_{L^{\infty}} \\
    \nonumber
    & ~~~+\left\| \widetilde{\times}_{d-1} (x_1,\ldots,x_{d-1})x_d-x_1\cdots x_d \right\|_{L^{\infty}} \\ \nonumber
    &\leq 2^{-2KW-1}+\left\| \widetilde{\times}_{d-1} (x_1,\ldots,x_{d-1})-x_1\cdots x_{d-1} \right\|_{L^{\infty}}\nonumber\\
    \nonumber
    &\leq \cdots \leq \frac{d-1}{2}\cdot2^{-2KW}
\end{align}
 
	Then the 3D ReLU network implementing $\widetilde{\times}_d$ is presented in Figure \ref{Extend_prod}, which has depth $(d-1)K$, width $4(W+1)+d-2=4W+d+2$, and height $W$.
\end{proof}

\begin{figure}[htb!]
\center{\includegraphics[width=\linewidth] {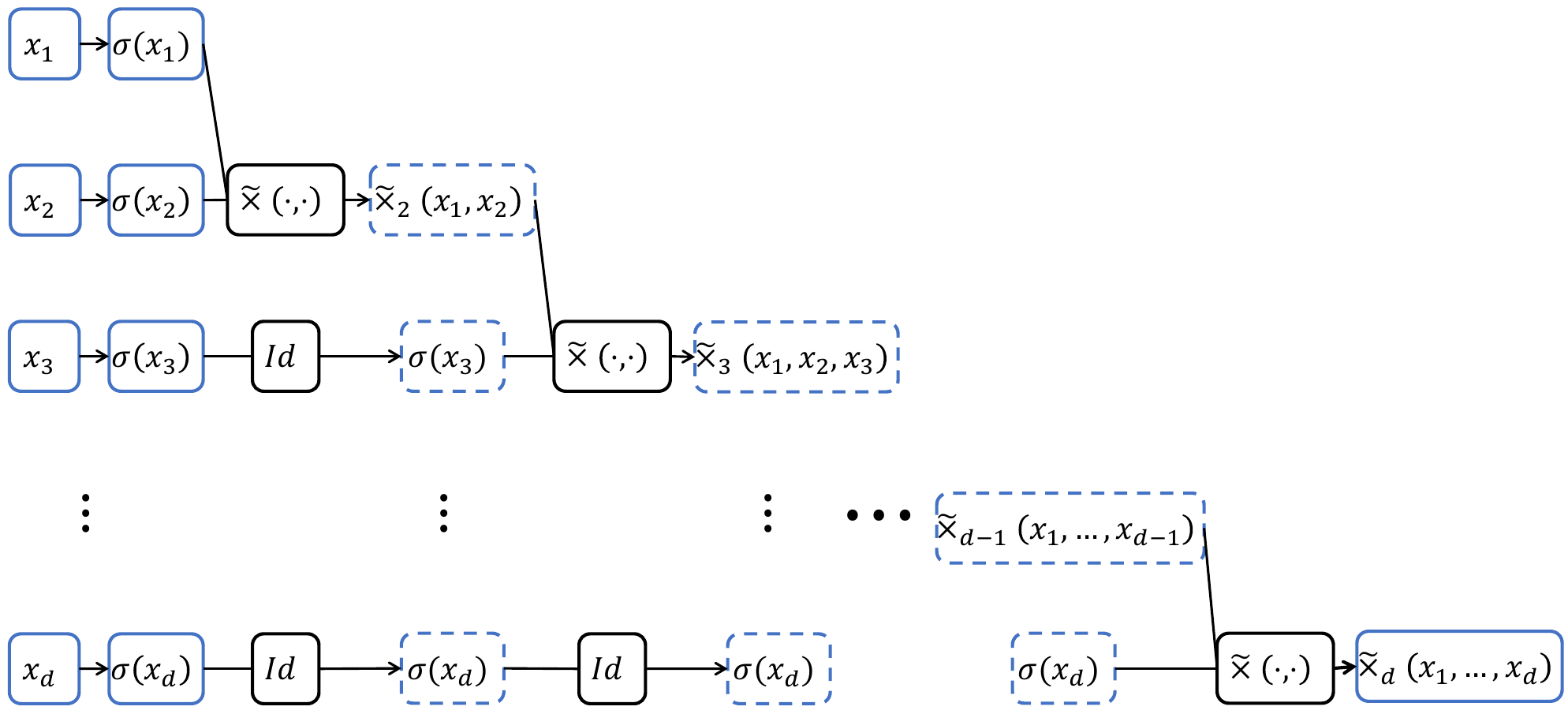}}
\caption{An illustration of a network $\mathcal{N}_{4W+d+2,K,W}$ that implements the product function.}
\label{Extend_prod}
\end{figure}




\textbf{Remark 4}. Our theory does not claim that height can replace depth and width. Instead, we advocate that intra-layer links are a powerful add-on to the network. The dimension of height naturally arises after introducing intra-layer links in a layer to create a hierarchy. In other words, the 2D-3D network transformation leads to efficacy in improving the network's approximation ability. More favorably, per our earlier analysis, the 2D-3D network transformation via inserting intra-layer links is not expensive in terms of both parametric and computational complexity.

\section{Experiments}

To evaluate if a 3D network can truly deliver good performance as theory suggests, we compare it with 2D networks such as fully-connected networks and ResNet on 5 synthetic datasets, 15 public tabular datasets, and 2 image benchmarks. 
Results show that on both regression and classification tasks, 3D networks can outperform the standard 2D network with fewer parameters, and adding height can also further enhance the performance of ResNet. At the same time, although 3D networks have more links to optimize, our experiments also demonstrate that 3D networks do not suffer the optimization issue, which aligns with what was observed in \cite{sadat2023connected}. Below are our experiments on image datasets.


Specifically, we first conduct regression experiments on 5 synthetic polynomials, 2 widely-used public datasets, and 3 real-world datasets. Results suggest that the proposed 3D network can boost the network's representation power. Second, we demonstrate the effectiveness and efficiency of the network on classification tasks using 8 tabular datasets, 2 fault diagnosis datasets, and 3 image benchmarks (CIFAR100, Tiny-ImageNet, and ImageNet). The experiments are implemented in Tensorflow using a CPU Intel i7-11800H processor at 2.3Hz and a GPU NVIDIA T600.

\subsection{Regression}

\subsubsection{Results on Synthetic Datasets}

We synthesize five elementary polynomials whose expressions are shown in Table~\ref{Table:Synthetic}. The normally distributed noise with 0 mean and 0.5 variance is added to the synthesized signals. A total of 1,000 points are sampled from $[-3,3]$ with an equal distance for training. The width of each layer is 8, and the depths of each network are respectively set to $[3,5,11,11,11]$. The epoch is 200, the batch size is 128, and ‘Adam' \cite{kingma2014adam} is the optimizer. The schedule for the learning rate is ‘ReduceLROnPlateau'. To test the generalization accuracy of the model, 1,000 points are sampled from $[-4,4]$, none of which appears in the training. Comparisons between fully-connected networks and the proposed 3D network are shown in Table~\ref{Table:MSE Value of Synthetic Experiments} and Figure~\ref{fig:polynomial function}.

\begin{table}[!t]
\centering\footnotesize
\caption{The expressions of synthetic polynomials.}
\label{Table:Synthetic}
\renewcommand{\arraystretch}{1.2}
\setlength{\tabcolsep}{5pt}
\begin{tabular}{c|l}
\hline
Synthetic functions & Expression  \\
\hline
$p_1(x)$ & $x^2+x$\\
$p_2(x)$ & $x^3+x^2+x$\\
$p_3(x)$ & $x^4+x^3+x^2+x$\\
$p_4(x)$ & $x^5+x^4+x^3+x^2+x$\\
$p_5(x)$ & $x^6+x^5+x^4+x^3+x^2+x$\\
\hline
\end{tabular}
\end{table}

\begin{table*}[!t]
\centering\footnotesize
\caption{MSE values of the proposed 3D network and fully-connected networks on synthetic experiments. \#PRM denotes the number of parameters.}
\label{Table:MSE Value of Synthetic Experiments}
\renewcommand{\arraystretch}{1.2}
\setlength{\tabcolsep}{5pt}
\begin{tabular}{c|c|c|c|c}
\hline
 Synthetic functions& Indicators & Three dimensional & Fully-connected ($s_1$) & Fully-connected ($s_2$)\\
\hline
\multirow{2}*{$p_1(x)$} & MSE & 0.6616 & 1.5755 & 1.0700 \\
\cline{2-5}
 & \#PRM & 99 & 97 & 169 \\
\hline
\multirow{2}*{$p_2(x)$} & MSE & 43.6248 & 46.9811 & 45.4330 \\
\cline{2-5}
& \#PRM & 245 & 241 & 313 \\
\hline
\multirow{2}*{$p_3(x)$} & MSE & 432.0960 & 491.0851 & 439.0270 \\
\cline{2-5}
& \#PRM & 683 & 673 & 745 \\
\hline
\multirow{2}*{$p_4(x)$} & MSE & 17,204.2516 & 23,397.8515 & 17,936.0898 \\
\cline{2-5}
& \#PRM & 829 & 817 & 889 \\
\hline
\multirow{2}*{$p_5(x)$} & MSE & 304,183.1563 & 334,655.9375 & 307,388.1875 \\
\cline{2-5}
& \#PRM & 829 & 817 & 889 \\
\hline
\end{tabular}
\vspace{-0.6cm}
\end{table*}

\begin{table}[ht!]
\centering\footnotesize
\caption{Statistics of 5 regression tabular datasets.}
\label{Table:Statistics of regression tabular datasets}
\renewcommand{\arraystretch}{1.2}
\setlength{\tabcolsep}{5pt}
\begin{tabular}{c|c|c|c}
\hline
Datasets & Instances  & Features & Features Type  \\
\hline
\textit{Boston housing} & 506    & 13  & discrete  \\
\hline
\textit{California housing} & 20640   & 8  & discrete  \\
\hline
\textit{Walmart}  & 97056   & 20 &  discrete \\
\hline
\textit{Energy consumption}  & 35024    &  6 &  continuous \\
\hline
\textit{Wind power}  & 49166   & 5  & continuous  \\
\hline
\end{tabular}
\vspace{-0.5cm}
\end{table}

In Table~\ref{Table:MSE Value of Synthetic Experiments}, we quantitatively compare the approximation errors between the proposed 3D network and fully-connected networks using the Mean Squared Error (MSE). We verify two kinds of fully-connected networks: one ($s_1$) has the same width and depth as proposed 3D networks; the other ($s_2$) is deeper. Thus, $s_1$ has comparable parameters as 3D networks and $s_2$ has more. We can draw two highlights from Table~\ref{Table:MSE Value of Synthetic Experiments}. First, when the parameters are comparable, 3D networks can lead to a huge improvement in MSE compared with fully-connected networks ($s_1$). For example, in approximating $p_1(x)$, the MSE decreases by 58$\%$. Second, when the fully-connected networks go deeper ($s_2$), 3D networks can still perform better.

Figure~\ref{fig:polynomial function} shows qualitatively the generalization between 3D and fully-connected ($s_1$) networks. We find that in $[-3,3]$ both 3D and fully-connected networks agree with the original functions well. However, in approximating peripheral parts of the function like $[-4,-3]$ and $[3,4]$, 3D networks outperform fully-connected ($s_1$) networks.




\begin{figure*}[!t]
\centering    
\subfloat[]{
\includegraphics[width=2.9cm,height=1.8cm]{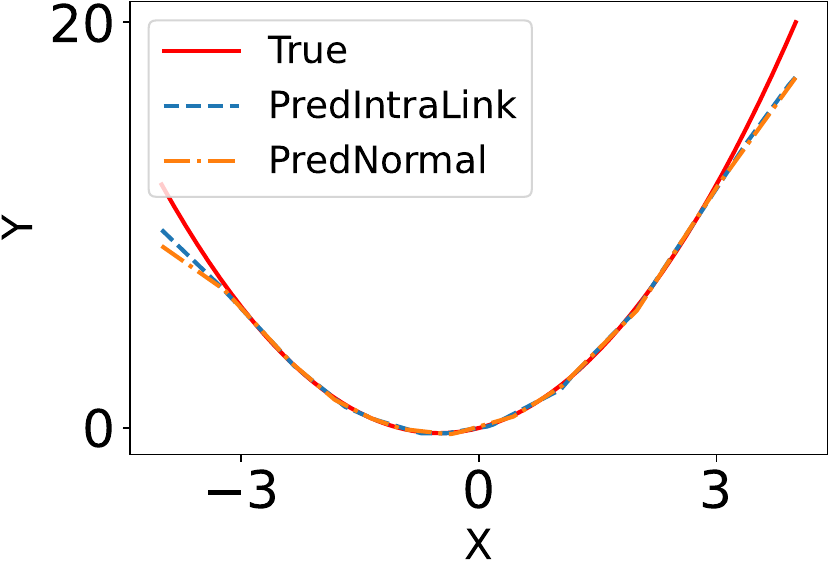}
\label{fig:/Mathmatic Regression x_hat_2} }~
\subfloat[]{
\includegraphics[width=2.9cm,height=1.8cm]{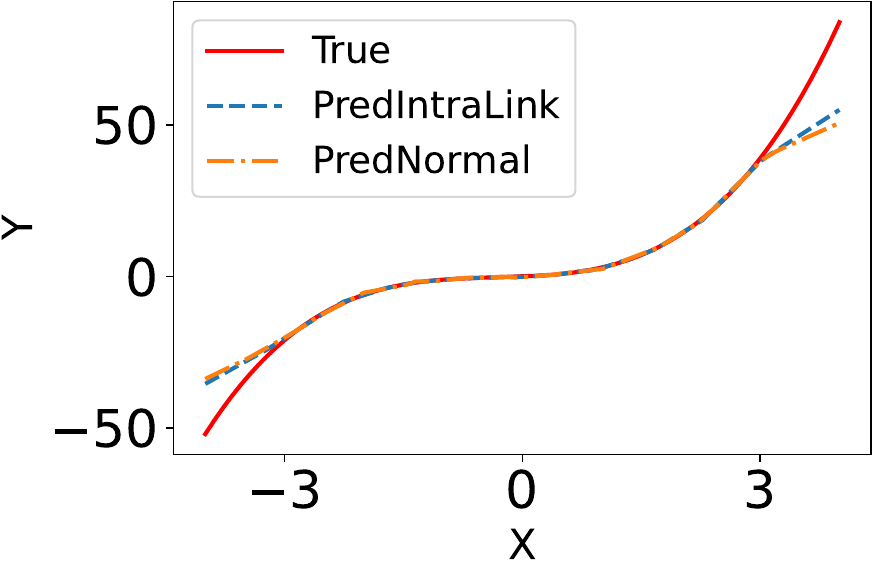}
\label{fig:/Mathmatic Regression x_hat_3} }~
\subfloat[]{
\includegraphics[width=2.9cm,height=1.8cm]{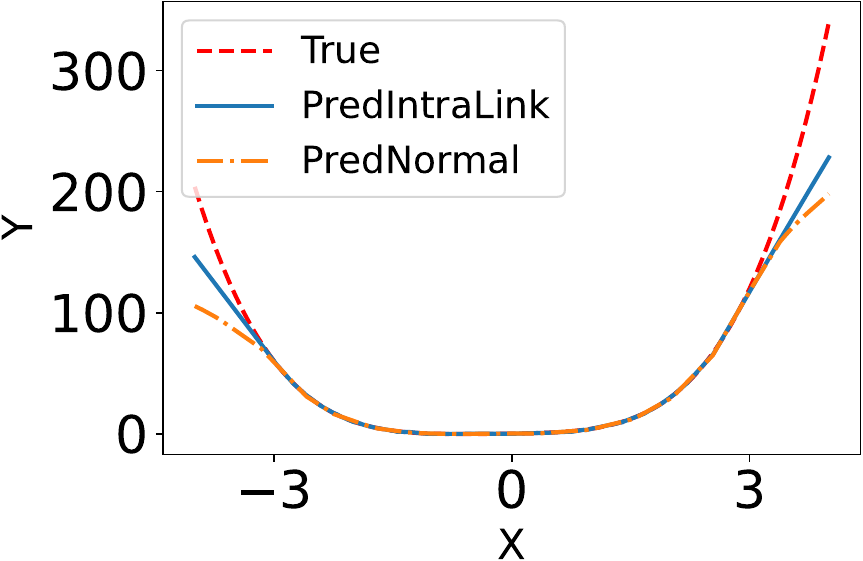}
\label{fig:/Mathmatic Regression x_hat_4} }~
\subfloat[]{
\includegraphics[width=2.9cm,height=1.8cm]{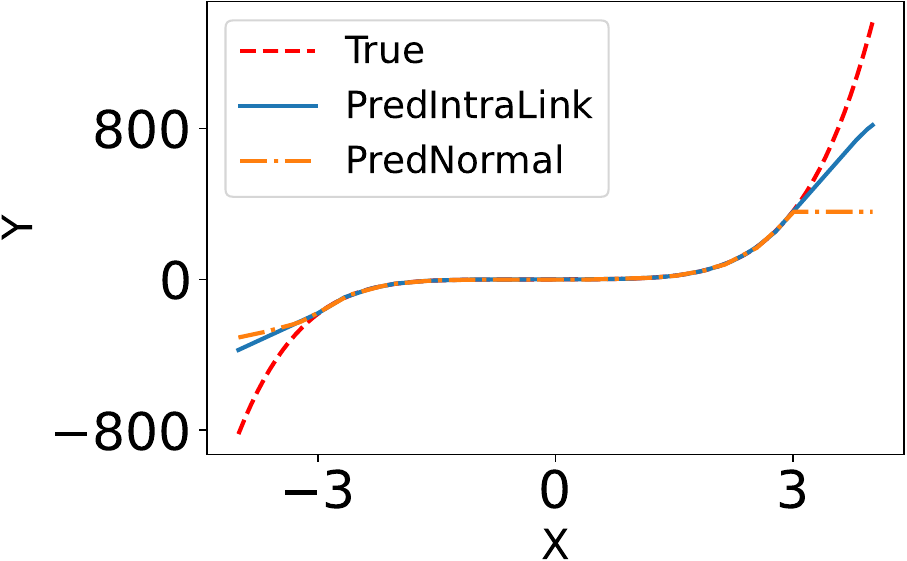}
\label{fig:/Mathmatic Regression x_hat_5} }~
\subfloat[]{
\includegraphics[width=2.9cm,height=1.8cm]{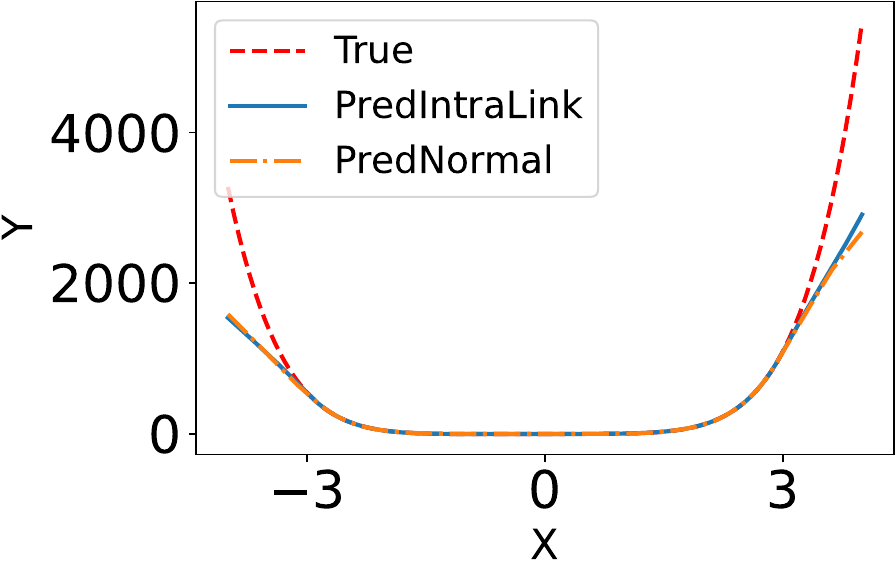}
\label{fig:/Mathmatic Regression x_hat_6} }~
\caption{Polynomial function. (a)$\sim$(e) are separately the polynomial functions from the second order to the sixth order.  }
\label{fig:polynomial function}
\vspace{-0.5cm}
\end{figure*}

\subsubsection{Results on Real-world Datasets}

\begin{figure*}[!t]
\centering    
\subfloat[]{
\includegraphics[width=2.9cm,height=1.8cm]{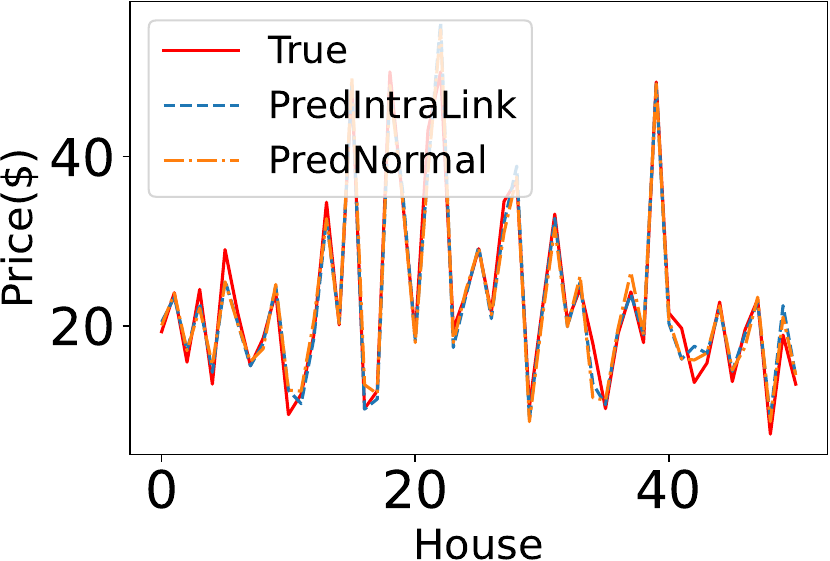}
\label{fig:/Boston house price2}}~
\subfloat[]{
\includegraphics[width=2.9cm,height=1.8cm]{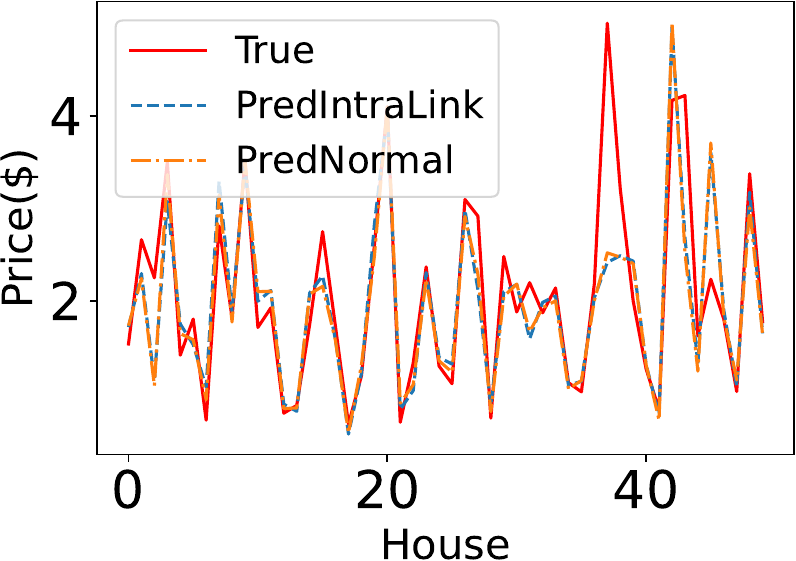}
\label{fig:/California house price3} } ~
\subfloat[]{
\includegraphics[width=2.9cm,height=1.8cm]{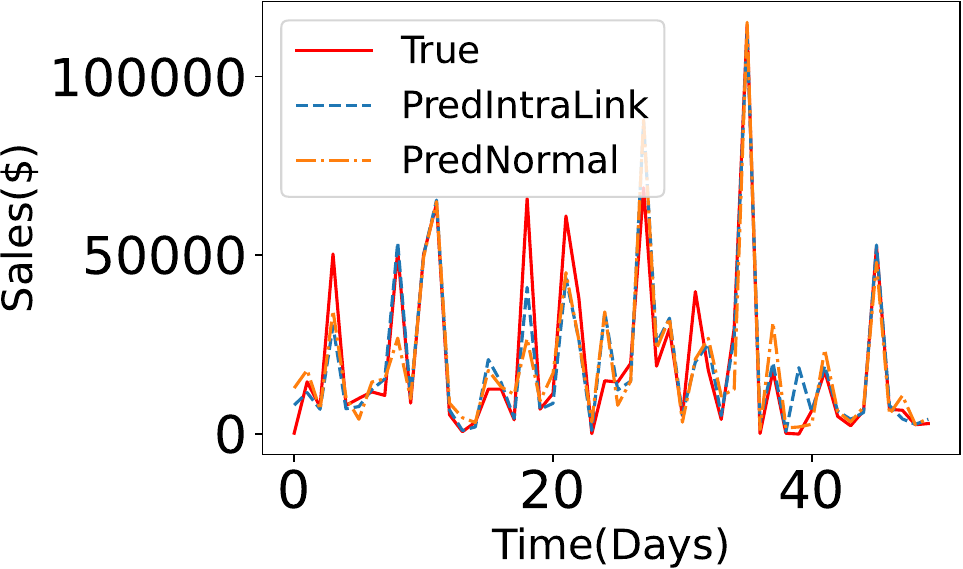}
\label{fig:/Walmart Regression} }~
\subfloat[]{
\includegraphics[width=2.9cm,height=1.8cm]{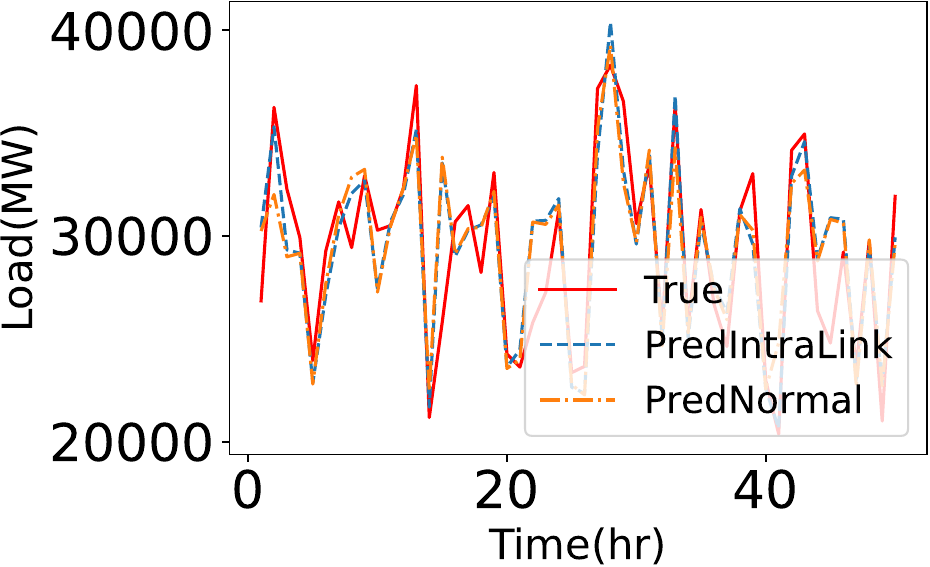}
\label{fig:/Load forecasting} }~
\subfloat[]{
\includegraphics[width=2.9cm,height=1.8cm]{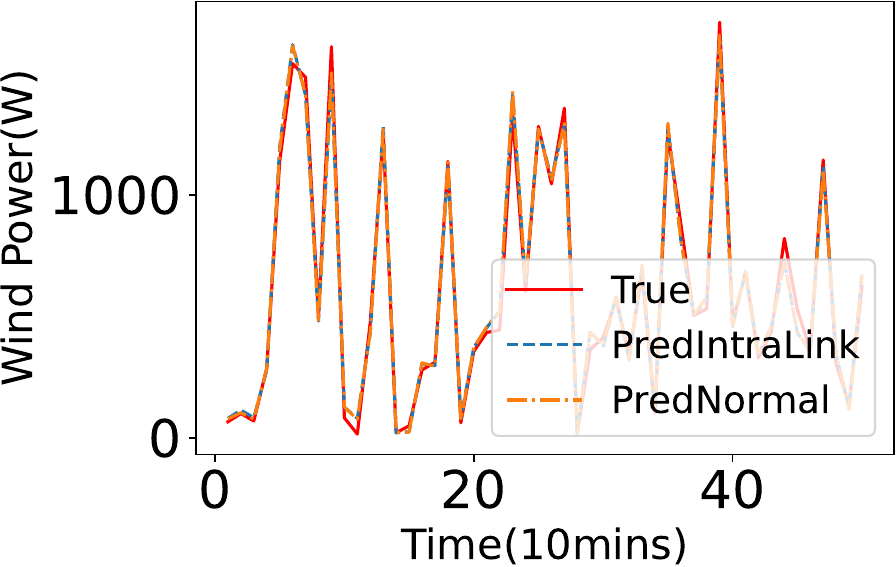}
\label{fig:/Wind forecasting} }~
\caption{Real-world regression experiments. (a) Boston house price; (b) California house price; (c) Walmart sales forecasting; (d) Load consumption forecasting; (e) Wind power forecasting. }
\label{fig:Real-world Regression Experiments}
\vspace{-0.5cm}
\end{figure*}

Furthermore, encouraged by positive results on synthetic experiments, we continue the regression experiments on 5 widely-used real-world datasets: \textit{Boston housing}\footnote{https://archive.ics.uci.edu/ml/machine-learning-databases/housing/housing.data}, \textit{California housing}\footnote{https://archive.ics.uci.edu/ml/machine-learning-databases/housing/housing.data}, \textit{Walmart}\footnote{https://www.kaggle.com/c/walmart-recruiting-store-sales-forecasting}, \textit{Energy consumption}\footnote{https://www.kaggle.com/datasets/robikscube/hourly-energy-consumption}, \textit{Wind power}\footnote{https://www.kaggle.com/datasets/theforcecoder/wind-power-forecasting}. The statistics of these datasets are summarized in Table~\ref{Table:Statistics of regression tabular datasets}. For small datasets, they are split into training and test sets with a ratio of 0.9:0.1. For the rest large datasets, we split them with a ratio of 0.8:0.1:0.1. Each layer has 256 neurons, and the depths of each network are respectively $[2,2,8,15,10]$. The epoch is 500, and the batch size is 128. Other hyper-parameters are the same as synthetic experiments. 

Consistent with synthetic experiments, Tabel~\ref{Table:MSE Value of Real-world regression experiments} summarizes the MSE and parameters of each network. $s_1$ has the same width and depth as the 3D networks, and $s_2$ has one more layer. We also find the 3D networks consistently outperform fully-connected networks ($s_1$ and $s_2$) on the 5 real-world datasets. When 3D and fully-connected networks have comparable parameters, 3D networks can lead by a large margin.

Figure~\ref{fig:Real-world Regression Experiments} shows visually the regression results between 3D and fully-connected networks on 5 datasets. By adding an intra-layer within each layer, the regression performance has significant improvement in real-world regression tasks especially at peaks. For example, the Walmart sales predicted by the 3D network at Day 10, and 19 are closer to the actual sales. For the load consumption prediction, the 3D network generates curves that align with the true load better, especially at 3hr and 42hr. 


In brief, in regression tasks, the 3D networks take the lead by a large margin. Such a superiority corroborates our theoretical analysis that adding intra-layer links can boost the network's representation power.

\begin{table}[!t]
\centering\footnotesize
\caption{MSE values of 3D and 2D networks on real-world datasets. \#PRM denotes the number of parameters.}
\label{Table:MSE Value of Real-world regression experiments}
\renewcommand{\arraystretch}{1.2}
\setlength{\tabcolsep}{5pt}
\begin{tabular}{c|c|c|c|c}
\hline
 Datasets & Indicators & 3D & 2D ($s_1$) & 2D ($s_2$)\\
\hline
\multirow{2}*{\textit{Boston house}} & MSE & 5.3113 & 10.5220 & 5.9844 \\
\cline{2-5}
 & \#PRM & 69,635 & 69,633 & 135,425 \\
\hline
\multirow{2}*{\textit{California house}} & MSE & 0.2518 & 0.2664 & 0.2631 \\
\cline{2-5}
& \#PRM &  69,635 & 69,633 & 135,425  \\
\hline
\multirow{2}*{\textit{Walmart sales}} & MSE & 88,556,472 & 111,667,784 & 110,398,056 \\
\cline{2-5}
& \#PRM & 466,185 & 466,177 & 531,969 \\
\hline
\multirow{2}*{\textit{load consumption}} & MSE & 7,922,336 & 8,191,562 & 8,191,562 \\
\cline{2-5}
& \#PRM & 923,152 & 923,137  & 988,929 \\
\hline
\multirow{2}*{\textit{Wind power}} & MSE & 3,812.82 & 3,991.2966 & 3,838.5940  \\
\cline{2-5}
& \#PRM & 593,931 & 593,921 & 659,713 \\
\hline
\end{tabular}
\end{table}

\subsection{Classification}

\subsubsection{Tabular Datasets}

We first investigate the effectiveness of 3D networks on classification tasks using tabular datasets. The tabular datasets contain \textit{Gaussian quantiles}\footnote{https://scikit-learn.org/stable/modules/generated/sklearn.datasets.make-gaussian-quantiles.html}, \textit{Breast cancer}\footnote{https://archive.ics.uci.edu/ml/datasets/Breast+cancer+Wisconsin+(Diagnostic)}, \textit{kddcup99}\footnote{http://kdd.ics.uci.edu/databases/kddcup99/kddcup99.html}, \textit{Wine}\footnote{https://scikit-learn.org/stable/modules/generated/sklearn.datasets.load wine.html}, \textit{Banknote authentication}\footnote{http://archive.ics.uci.edu/ml/datasets/banknote+authentication}, \textit{Heart Failure}\footnote{https://archive.ics.uci.edu/ml/datasets/heart+disease}, \textit{Ionosphere}\footnote{http://archive.ics.uci.edu/dataset/52/ionosphere}, \textit{Mobile Price}\footnote{https://www.kaggle.com/datasets/iabhishekofficial/mobile-price-classification}, \textit{CWRU}\footnote{https://engineering.case.edu/bearingdatacenter/apparatus-and-procedures}, \textit{Motor fault}\footnote{https://gitlab.com/power-systems-technion/motor-faults}. They are publicly available from the Python scikit-learn package, UCI machine learning repository, Kaggle, and so on. The \textit{Concentric circles} and \textit{Gaussian quantiles} are two synthetic datasets, and the rest are all real-world datasets including medical dataset, network intrusion detection dataset (\textit{Kddcup99}), climate dataset (\textit{Ionosphere}), mobile price dataset and fault diagnosis dataset \textit{etc.}. All the statistics of tabular datasets are summarized in Table~\ref{Table:Statistics of tabular datasets}.



\begin{table}[!t]
\centering\footnotesize
\caption{Statistics of 10 classification tabular datasets.}
\label{Table:Statistics of tabular datasets}
\renewcommand{\arraystretch}{1.2}
\setlength{\tabcolsep}{5pt}
\begin{tabular}{c|c|c|c|c}
\hline
Datasets & Instances & Classes & Features & Features Type  \\
\hline
\textit{Gaussian quantiles}  & 10000  & 2  & 2  & discrete  \\
\hline
\textit{Breast cancer}  &  569 & 2  &  30 & discrete  \\
\hline
\textit{kddcup99}  &  494021 &  23 & 41  & discrete  \\
\hline
\textit{Concentric circles}  &  16000 &  4  & 2  &  discrete \\
\hline
\textit{Banknote authentication}  & 1347  & 2  & 4  & discrete  \\
\hline
\textit{Heart Failure}  & 299  &  2 & 12  & discrete  \\
\hline
\textit{Ionosphere}  & 350  & 2  & 34  & discrete  \\
\hline
\textit{Mobile Price}  &  2000 & 4  &  20 & discrete  \\
\hline
\textit{CWRU}  &  2400 &  10 & 1200  & continuous  \\
\hline
\textit{Motor fault}  & 3392  & 6  & 120  & continuous  \\
\hline
\end{tabular}
\end{table}

Similar to regression experiments, Table~\ref{Table:Test accuracy of all datasets} indicates as well that 3D networks have stronger expressivity than fully-connected networks. When evaluated on simple datasets such as \textit{concentric circles}, \textit{Gaussian quantiles}, and \textit{Kddcup99}, the improvement is moderate. This is because both 3D and fully-connected networks can extract features for accurate classification. In contrast, the 3D networks perform much better on complex datasets. For example, for the fault diagnosis in \textit{CWRU}, the gains are respectively 2.04$\%$ and 1.2$\%$. These improvements in test accuracy substantiate the effectiveness of 3D networks.


\begin{table*}[!t]
\centering\footnotesize
\caption{Test accuracy of 3D and fully-connected networks on tabular classification datasets. \#PRM denotes the number of parameters.}
\label{Table:Test accuracy of all datasets}
\renewcommand{\arraystretch}{1.2}
\setlength{\tabcolsep}{5pt}
\begin{tabular}{c|c|c|c|c}
\hline
 Datasets & Indicators & Intra-layer & Fully-connected ($s_1$) & Fully-connected ($s_2$)\\
\hline
\multirow{2}*{\textit{Concentric circles}} & ACC & 99.93$\%$ & 99.75$\%$ & 99.43$\%$ \\
\cline{2-5}
 & \#PRM & 61 & 60  &  132 \\
\hline
\multirow{2}*{\textit{Gaussian quantiles}} & ACC & 97.2$\%$ & 96.6$\%$ & 96.1$\%$ \\
\cline{2-5}
& \#PRM & 554  & 546 &  618 \\
\hline
\multirow{2}*{\textit{Breast cancer}} & ACC & 98.246$\%$ & 96.49$\%$ & 96.86$\%$ \\
\cline{2-5}
& \#PRM & 2,380,812 & 2,380,802 & 2,643,458 \\
\hline
\multirow{2}*{\textit{Kddcup99}} & ACC & 99.95$\%$ & 99.94$\%$ & 99.945$\%$ \\
\cline{2-5}
& \#PRM & 295,961 & 295,959 & 558,615 \\
\hline
\multirow{2}*{\textit{Banknote authentication}} & ACC & 100$\%$ & 98.52$\%$ & 99.26$\%$  \\
\cline{2-5}
& \#PRM & 3,587 & 3,586 & 266,242 \\
\hline
\multirow{2}*{\textit{Heart Failure}} & ACC & 83.33$\%$ & 80.00$\%$  & 76.67$\%$ \\
\cline{2-5}
 & \#PRM & 795,654 & 795,650 & 1,058,306 \\
\hline
\multirow{2}*{\textit{Ionosphere}} & ACC & 100$\%$ & 97.14$\%$ & 94.28$\%$ \\
\cline{2-5}
& \#PRM & 806,918  & 806,914 & 1,069,570  \\
\hline
\multirow{2}*{\textit{Mobile Price}} & ACC & 92.5$\%$ & 91.5$\%$ & 92$\%$ \\
\cline{2-5}
& \#PRM &  806,918  & 806,914 & 1,069,570   \\
\hline
\multirow{2}*{\textit{CWRU}} & ACC & 89.12$\%$ & 87.08$\%$ & 87.92$\%$ \\
\cline{2-5}
& \#PRM & 2,105,873 & 2,105,866 & 2,368,522 \\
\hline
\multirow{2}*{\textit{Motor fault}} & ACC & 98.53$\%$ & 97.35$\%$ & 97.64$\%$  \\
\cline{2-5}
& \#PRM & 2,166,287 & 2,166,278 & 2,428,934 \\
\hline
\end{tabular}
\end{table*}

\begin{table}[!t]
\centering\footnotesize
\caption{Test accuracy of 3D and 2D networks on two image datasets. \#PRM denotes the number of parameters.}
\label{Table:Test accuracy of Images}
\renewcommand{\arraystretch}{1.2}
\setlength{\tabcolsep}{5pt}
\begin{tabular}{c|c|c|c}
\hline
 Datasets & Indicators & Intra-layer & ResNet18  \\
\hline
\multirow{3}*{CIFAR100} & ACC &  75.34$\%$ & 75.61$\%$ \\
\cline{2-4}
 & \#PRM & 6.55M  & 11.2M   \\
 \cline{2-4}
 & FLOPS & 510.57M & 611.12M \\
 \hline
\multirow{2}*{Tiny-ImageNet} & ACC & 46.60$\%$ & 42.51$\%$ \\
\cline{2-4}
 & \#PRM & 5.07M  & 11.2M   \\ 
 \cline{2-4}
 & FLOPS & 656.18M & 756.84M \\
\hline
\end{tabular}
\end{table}

\begin{figure}[h]
    \centering
    \includegraphics[width=0.6\linewidth]{ 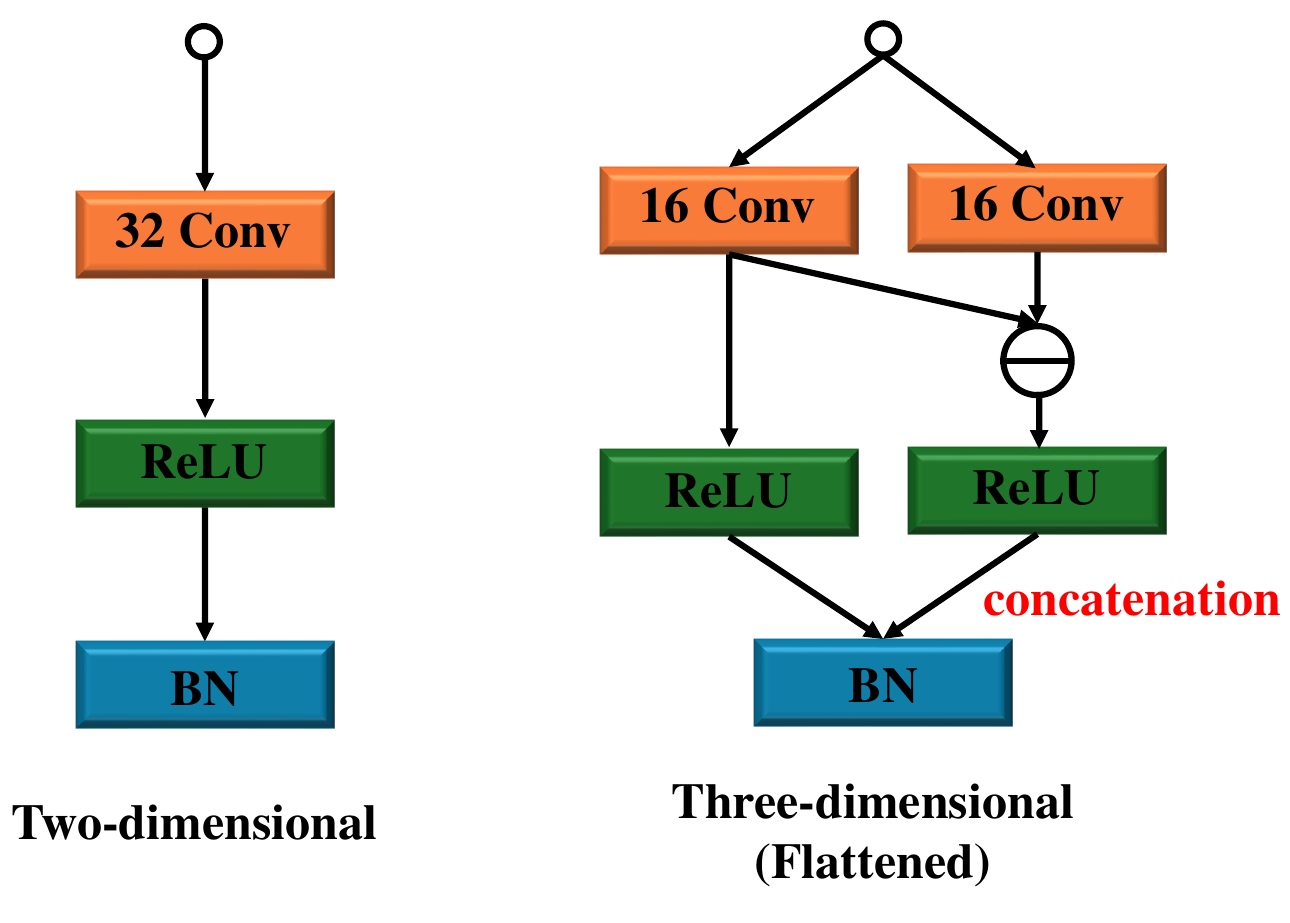}
    \caption{A visualization of residual blocks used in 2D and 3D networks.}
    \label{fig:residual_block}
    \vspace{-0.3cm}
\end{figure}

\subsubsection{Image Datasets} We conduct three image classification experiments utilizing CIFAR100~\cite{krizhevsky2009learning}, Tiny-ImageNet~\cite{le2015tiny}, and ImageNet~\cite{deng2009imagenet}. In each layer, channels are equally divided into two parts, and each part is used to capture the image features. The outputs of one part are added to the other in the form of shortcuts by trainable parameters and then concatenated (height=2). We visualize residual blocks used in two- and 3D networks in Figure \ref{fig:residual_block}. Mathematically, suppose that the input of a height layer is $[a,b]$, while the output is $[\mathrm{ReLU}(a), \mathrm{ReLU}(\mathrm{ReLU}(a)+b)]$. The 3D network contains 1 convolutional layer, several height layers, and 1 fully-connected layer with ‘softmax’. We use the same hyperparameters as ResNet18. 

Due to the limitation of memory, we made some simplifications to the experiments. For CIFAR100 classification, we augment the dataset by rotating, shifting, shearing, and horizontally flipping the original images. However, for TinyImageNet classification we ignored the data augment to save memories. All models are trained with a batch size of 128 using the ‘SGD’ optimizer with an initialized learning rate of 0.1 and moment 0.9. 

Taking ResNet18 as a benchmark, Table~\ref{Table:Test accuracy of Images} shows that when we use fewer convolutional kernels in a 3D network, it has similar performance on CIFAR100 but significantly better performance on Tiny-ImageNet compared to ResNet18, which verifies that adding height can enhance the network’s representation power. In ImageNet, turning ResNet into a 3D network is also an effective strategy. As Table \label{ImageNetResults_small} shows, the 3D network with 7M parameters has the comparable performance with the DY-ResNet-10 that has 18.6 parameters. 

\begin{table}[htb]
 \centering

\caption{The top-1 error (\%) comparisons on ImageNet validation set.}
 \begin{tabular}{ p{3.5cm}|p{1cm}|p{1.5cm}  }
 \hline
 Network  & Params & Error ($\%$) \\
  \hline
  DY-ResNet-10 (2019) \cite{chen2020dynamic}  & 18.6M & 32.3\\
  \hline
  ResNet-18 (Intra-linked) &  7.0M & 33.3\\
  \hline
   \end{tabular}
 \label{ImageNetResults_small}
 \vspace{-2mm}
\end{table}   




\section{Conclusion}

In this draft, we propose a 3D network by interconnecting neurons of a layer in a wide and deep network. Furthermore, via bound estimation, dedicated construction, and approximation error analysis, we have shown that a 3D network is much more expressive than a fully-connected one, given the same number of neurons. Then, we have shown that 3D networks can deliver superior performance via systematic experiments. Future endeavors can be using 3D networks to solve more real-world problems and find their killer applications.



\bibliographystyle{ieeetr}
\bibliography{reference}

\end{document}


\title{\huge\bf Benefits of Height in Neural Networks}



\maketitle


%
\IEEEpeerreviewmaketitle



\def\N{\mathcal N}
\def\O{\mathcal O}
\def\mN{\mathcal N}
\def\mP{\mathcal P}
\def\mR{\mathcal R}
\def\mA{\mathcal A}
\def\mB{\mathcal B}
\def\0{\mathbf{0}}
\def\v{\mathbf{v}}
\def\x{\mathbf{x}}
\def\p{\mathbf{p}}
\def\a{\mathbf{a}}
\def\tf{\tilde{f}}
\def\S{\mathcal{S}}

\section{Proof of Theorem \ref{thm:intra_bound_HD}}
\label{sec:proof_HD}

\begin{theorem}[Upper bound of two-dimensional networks \cite{montufar2017notes} for $n$-dimensional inputs]
    Let $f: \mathbb{R}^n \rightarrow \mathbb{R}$ be a PWL function represented by an $\mathbb{R}^n \rightarrow \mathbb{R}$  ReLU DNN with depth  $K$  and widths  $w_{1}, \ldots, w_{K}$  of  $K$  hidden layers. Then $f$ has at most  $\prod_{k=1}^{K}\sum_{j=0}^n \binom{w_k}{j}$ linear regions. 
\end{theorem}

\begin{theorem}[Upper bound of three-dimensional networks for $n$-dimensional inputs]
\label{thm:intra_bound_HD}
Let $f: \mathbb{R}^n \rightarrow \mathbb{R}$ be a PWL function represented by an $\mathbb{R}^n \rightarrow \mathbb{R}$  ReLU DNN with height=2 in each hidden layer, depth $K$, and widths  $(w_{11},w_{12}), \ldots, (w_{K1},w_{K2})$ of $K$ hidden layers, where $w_{k1}=w_{k2}=w_k$ for $k=1,\ldots,K$. Then $f$ has at most $\prod_{k=1}^{K}\sum_{j=0}^n \binom{3w_{k}+1}{j}$ linear regions. 
\end{theorem}

\begin{lemma}[Zaslavsky's Theorem \cite{zaslavsky1975facing,Stanley04anintroduction}]\label{thm:ZaslavskyNN}
Let $\mA = \left\{H_i\subset V:1\leq i\leq m\right\}$ be an arrangement in
    $\mathbb{R}^{n}$. Then, the number of regions for the arrangement $\mA$ satisfies
    \begin{eqnarray} \label{eq:region_general1}
        r(\mA)\leq\sum_{i=0}^{n} \binom{m}{i}.
\end{eqnarray} 
\end{lemma}

\begin{proof}
We prove by induction on  $K$. For the base case  $K=1$, $\tilde{f}_{1}^{(2i-1)}=\sigma\left(\tilde{g}_{1}^{(2i-1)}\right)$ produces one hyperplane in the input space $\mathbb{R}^n$. Furthermore, $\tilde{f}_{1}^{(2i)}=\sigma\left(\tilde{g}_{1}^{(2i)}-\tilde{f}_{1}^{(2i-1)}\right) = \sigma\left(\tilde{g}_{1}^{(2i)}-\sigma\left(\tilde{g}_{1}^{(2i-1)}\right)\right)$ produces at most two hyperplanes in the input space $\mathbb{R}^n$. Therefore, in total, the $2w_1$ neurons in the first layer produce $(1+2)\cdot w_1=3w_1$ hyperplanes in the input space $\mathbb{R}^n$. Then by Zaslavsky's Theorem, it will produce at most $\sum_{j=0}^n \binom{3w_{1}+1}{j}$ linear regions in the input space $\mathbb{R}^n$. 
For the induction step, we assume that for some  $K \geq 1$, any  $\mathbb{R}^n \rightarrow \mathbb{R}$  ReLU DNN with every two neurons linked in each hidden layer, depth  $K$  and widths  $w_{1}, \ldots, w_{K}$  of  $K$  hidden layers produces at most  $\prod_{k=1}^{K}\sum_{j=0}^n \binom{3w_{k}+1}{j}$ linear regions.  Now we consider any  $\mathbb{R}^n \rightarrow \mathbb{R}$  ReLU DNN with every two neurons linked in each hidden layer, depth  $K+1$  and widths  $w_{1}, \ldots, w_{K+1}$  of  $K+1$ hidden layers. Then for each linear region $S$ produced by the first $K+1$ layers, again,  $\tilde{f}_{K+1}^{(2i-1)}=\sigma\left(\tilde{g}_{K+1}^{(2i-1)}\right)$ produces one hyperplane in $S$. Furthermore, $\tilde{f}_{K+1}^{(2i)}=\sigma\left(\tilde{g}_{K+1}^{(2i)}-\tilde{f}_{K+1}^{(2i-1)}\right) = \sigma\left(\tilde{g}_{K+1}^{(2i)}-\sigma\left(\tilde{g}_{K+1}^{(2i-1)}\right)\right)$ produces at most two hyperplanes in the  $S$. Therefore, in total, the $2w_{K+1}$ neurons in the $K+1$ layer produces $(1+2)\cdot w_{K+1}=3w_{K+1}$ hyperplanes in  $S$. Then by Zaslavsky's Theorem, it will produce at most $\sum_{j=0}^n \binom{3w_{K+1}+1}{j}$ linear regions in  $S$. Thus $f$  has at most  $\prod_{k=1}^{K+1}\sum_{j=0}^n \binom{3w_{k}+1}{j}$ linear regions, and we conclude the proof. 
\end{proof}

\bibliographystyle{ieeetr}
\bibliography{reference}